\documentclass{article}

\usepackage[preprint,nonatbib]{neurips_2021}

\usepackage[utf8]{inputenc} 
\usepackage[T1]{fontenc}    
\usepackage{hyperref}       
\usepackage{url}            
\usepackage{booktabs}       
\usepackage{amsfonts}       
\usepackage{nicefrac}       
\usepackage{microtype}      
\usepackage{xcolor}         

\usepackage{bm}
\usepackage{amsmath,amssymb}
\usepackage{amsthm}
\usepackage{color}
\usepackage{graphicx}
\usepackage{algorithm}
\usepackage{algorithmicx}
\usepackage[noend]{algpseudocode}
\usepackage{nccmath}
\usepackage{subcaption}

\usepackage{pgfplots}
\DeclareUnicodeCharacter{2212}{−}
\usepgfplotslibrary{groupplots,dateplot}
\usetikzlibrary{patterns,shapes.arrows}
 \pgfplotsset{
   compat=1.7,   
 }

\newtheorem{thm}{Theorem}
\newtheorem{theorem}[thm]{Theorem}
\newtheorem{lemma}[thm]{Lemma}
\newtheorem{corollary}[thm]{Corollary}
\newtheorem{proposition}[thm]{Proposition}

\newtheorem{claim}[thm]{Claim}

\newcommand*\Receive[1]{\State \textbf{receive} #1 \;}
\newcommand*\Predict[1]{\State \textbf{predict} #1 \;}
\newcommand*\Init[1]{\State \textbf{init}: #1}
\newcommand*\algoInput[1]{\textbf{Input:} #1}
\newcommand*\algoOutput[1]{\textbf{Output:} #1}
\newcommand*\GETS{\!\gets\!}
\renewcommand{\thealgorithm}{\arabic{algorithm}\&2}

\newcommand{\algorule}[1][.4pt]{\par\vskip.5\baselineskip\hrule height #1\par\vskip.5\baselineskip}
\newcommand{\algrule}{\vspace{-.2\baselineskip}\algorule\vspace{-.2\baselineskip}}
\newcommand{\algotitle}[1]{\Statex\hspace*{-1.67em}{\footnotesize{$\triangleright$
      \texttt{#1}}}\vspace{0.2\baselineskip}}

\newcommand*\nexperts{n}
\newcommand*\nswitches{k}
\newcommand*\nsegments{\nswitches + 1}
\newcommand*\pool{m}  
\newcommand*\RT[1][T]{\mathcal{R}({\iit[1:{#1}]})}
\newcommand*\regret{\sumtT\!\lt-\!\sumtT\!\ltit}
\newcommand*\iit[1][t]{i_{#1}}

\newcommand*\yt[1][t]{y^{#1}}
\newcommand*\ythat[1][t]{\hat{y}^{#1}}
\newcommand*\pred{{\texttt{pred}}}
\newcommand*\predt{\pred(\bwt,\bxt)}

\newcommand*\bpi{\bm{\pi}}
\newcommand*\piw{\pi_{w}}
\newcommand*\pis{\pi_{s}}
\newcommand*\Pij[1][ij]{P_{#1}}

\newcommand*\Pww{P_{ww}}
\newcommand*\Pss{P_{ss}}
\newcommand*\Pws{P_{ws}}
\newcommand*\Psw{P_{sw}}
\newcommand*\chit[1][t]{\chi_{#1}}

\newcommand*\bat[1][t]{\bm{a}_{#1}}
\newcommand*\bst[1][t]{\bm{s}_{#1}}
\newcommand*\ati[1][t]{a^{#1}_{i}}
\newcommand*\sti[1][t]{s^{#1}_{i}}
\newcommand*\atj[1][t]{a^{#1}_{j}}
\newcommand*\stj[1][t]{s^{#1}_{j}}

\newcommand*\bu{\bm{u}}
\newcommand*\bw{\bm{w}}
\newcommand*\bv{\bm{v}}
\newcommand*\bei[1][i]{\bm{e}_{#1}}
\newcommand*\bx{\bm{x}}
\newcommand*\bwt[1][t]{\bm{w}^{#1}}
\newcommand*\but[1][t]{\bm{u}^{#1}}
\newcommand*\bvt[1][t]{\bm{v}^{#1}}
\newcommand*\bvtdt[1][t]{\tilde{\bm{v}}^{#1}}
\newcommand*\bwtdot[1][t]{\dot{\bm{w}}^{#1}}
\newcommand*\bwqdot[1][q]{\dot{\bm{w}}^{#1}}

\newcommand*\balpha{\bm{\alpha}}
\newcommand*\bbeta{\bm{\beta}}
\newcommand*\bxt[1][t]{\bm{x}^{#1}}

\newcommand*\bgammat[1][t+1]{\bm{\gamma}^{#1}}
\newcommand*\bp{\bm{p}}

\newcommand*\ui[1][i]{u_{#1}}
\newcommand*\wi[1][i]{w_{#1}}
\newcommand*\vi[1][i]{v_{#1}}
\newcommand*\xii[1][i]{x_{#1}}
\newcommand*\wti[1][t]{w^{#1}_{i}}
\newcommand*\wtj[1][t]{w^{#1}_{j}}
\newcommand*\wtk[1][t]{w^{#1}_{k}}
\newcommand*\uti[1][t]{u^{#1}_{i}}
\newcommand*\vti[1][t]{v^{#1}_{i}}

\newcommand*\vtk[1][t]{v^{#1}_{k}}
\newcommand*\xti[1][t]{x^{#1}_{i}}
\newcommand*\wtdoti[1][t]{\dot{w}^{#1}_{i}}
\newcommand*\wtdotj[1][t]{\dot{w}^{#1}_{j}}
\newcommand*\wtdotk[1][t]{\dot{w}^{#1}_{k}}
\newcommand*\lti[1][i]{\ell^{t}_{#1}}
\newcommand*\ltit[1][t]{\ell^{t}_{i_{#1}}}
\newcommand*\alphai[1][i]{\alpha_{#1}}

\newcommand*\betai[1][i]{\beta_{#1}}
\newcommand*\betati[1][t]{\beta^{#1}_{i}}

\newcommand*\betatk[1][t]{\beta^{#1}_{k}}
\newcommand*\gammatq[1][q]{\gamma^{t+1}_{#1}}
\newcommand*\pii[1][i]{p_{#1}}
\newcommand*\ri[1][i]{r_{#1}}
\newcommand*\rj{\ri[j]}

\newcommand*\RE[2]{D(#1, #2)}
\newcommand*\simplex{\Delta_{\nexperts}}
\newcommand*\entropyBig[1]{\cH\!\!\left(\!{#1}\!\right)}

\newcommand*\cC{\mathcal{C}}
\newcommand*\cO{\mathcal{O}}
\newcommand*\cH{\mathcal{H}}

\newcommand*\cP{\mathcal{P}}
\newcommand*\cL{\mathcal{L}}

\newcommand*\cY{\mathcal{Y}}
\newcommand*\cD{\mathcal{D}}
\newcommand*\Pinc{\mathcal{P}_{inc}}
\newcommand*\Pdec{\mathcal{P}_{dec}}
\newcommand*\Sinc{\mathcal{S}_{inc}}
\newcommand*\Sdec{\mathcal{S}_{dec}}

\newcommand*\bbetat[1][t]{\bbeta^{#1}}
\newcommand*\proj[2][\bmu]{\cP({#1};{#2})}
\newcommand*\projw[1][\Cb]{\cP(\bw;{#1})}
\newcommand*\Cb[1][\bbeta]{\cC({#1})}
\newcommand*\Cbt[1][t]{\cC({\bbetat[#1]})}
\newcommand*\alphasimplex{\simplex^{\alpha}}

\newcommand*\threshold{\phi}
\newcommand*\mW{\mathcal{W}}
\newcommand*\mL{\mathcal{L}}
\newcommand*\mM{\mathcal{M}}
\newcommand*\mH{\mathcal{H}}
\newcommand*\br{\bm{r}}
\newcommand*\podsth{PoDS-$\theta$}

\newcommand*\relint[1]{\textup{ri}\ {#1}}
\newcommand*\shareth{Share-$\theta$}

\newcommand*\nature{\texttt{nature}}

\newcommand*\learner{\texttt{learner}}
\newcommand*\Learner{\texttt{Learner}}

\newcommand*\onenorm[1]{\lVert{#1}\rVert_{1}}
\newcommand*\onenormBig[1]{\left\lVert{#1}\right\rVert_{1}}
\newcommand*\lt[1][t]{\ell^{#1}}
\newcommand*\sumin[1][i]{\sum_{{#1}=1}^{n}}
\newcommand*\sumtT{\sum_{t=1}^{T}}

\DeclareMathOperator*{\argmin}{arg\,min}

\title{Improved Regret Bounds for Tracking Experts with Memory}

\author{James Robinson\qquad Mark Herbster\\
  Department of Computer Science\\
  University College London\\
  London\\
  United Kingdom\\
  \texttt{\{j.robinson | m.herbster\}@cs.ucl.ac.uk}
}

\begin{document}

\maketitle

\begin{abstract}
  We address the problem of sequential {\em prediction with expert
    advice} in a non-stationary environment with long-term memory
  guarantees in the sense of Bousquet and Warmuth~\cite{bousquet2002tracking}.
  We give a linear-time algorithm that improves on the best known
  regret bounds~\cite{koolen2012putting}. This algorithm incorporates a
  relative entropy projection step. This projection is advantageous
  over previous weight-sharing approaches in that weight updates may
  come with implicit costs as in for example portfolio optimization.
  We give an algorithm to compute this projection step in linear time,
  which may be of independent interest.
\end{abstract}

\section{Introduction}\label{sec:introduction}
We consider the classic problem of online prediction with expert
advice~\cite{littlestone1994weighted} in a non-stationary environment.
In this model \nature\ sequentially generates outcomes which
\learner\ attempts to predict. Before making each prediction,
\learner\ listens to a set of $\nexperts$ experts who each make their
own predictions. \Learner\ bases its prediction on the advice of the
experts. After the prediction is made and the true outcome is revealed
by \nature, the accuracies of \learner's prediction and the expert
predictions are measured by a loss function. \Learner\ receives 
information on all expert losses on each trial. We make no
distributional assumptions about the outcomes generated, indeed
\nature\ may be assumed to be adversarial. The goal of \learner\ is to
predict well relative to a predetermined comparison class of
predictors, in this case the set of experts themselves. Unlike the
standard regret model, where \learner's performance is compared to the
single best predictor in hindsight, our aim is for \learner\ to
predict well relative to a sequence of comparison predictors. That is,
``switches'' occur in the data sequence and different experts are
assumed to  predict well at different times.

In this work our focus is on the case when this sequence consists of a
few unique predictors relative to the number of switches.
Thus most switches return to a previously ``good'' expert, and a
\learner\ that can exploit this fact by ``remembering'' the past can adapt
more quickly than a \learner\ who has no memory and must re-learn the
experts after every switch.
The problem of switching with memory in online learning is part of a
much broader and fundamental problem in machine learning: how a system
can adapt to new information yet retain knowledge of the past.
This is an area of research in many fields, including for example,
catastrophic forgetting in artificial neural
networks~\cite{french1999catastrophic,mccloskey1989catastrophic}. 

\paragraph{Contributions.}
In this paper we present an $\cO(\nexperts)$-time per trial projection-based algorithm 
for which we prove the best known regret bound for tracking experts
with memory. Our projection-based algorithm is intimately related to
a more traditional ``weight-sharing'' algorithm, which we show is a new method for {\em Mixing Past
  Posteriors} (MPP)~\cite{bousquet2002tracking}. We show that
surprisingly this method corresponds to the algorithm with the
previous best known regret bound for this problem~\cite{koolen2012putting}.
We also give an efficient $\cO(\nexperts)$-time algorithm for
computing exact relative entropy projection onto a simplex with
non-uniform (lower) box constraints. 
Finally, we provide a guarantee which favors projection-based updates
over weight-sharing updates when updating weights may incur costs.


The paper is organized as follows. We first  introduce
the model and discuss related work, giving a detailed
overview of the previous results on which we improve. In
Section~\ref{sec:pods} we give our main results, a regret bound which holds for two algorithms, and an algorithm to
compute relative entropy projection with non-uniform lower box
constraints in linear time. In Section~\ref{sec:mpp-mixing-scheme} we
derive a new ``geometric-decay'' method for MPP, and show the correspondence
to the current best known algorithm~\cite{koolen2012putting}. We
give a few concluding remarks in Section~\ref{sec:discussion}. All
proofs are contained in the appendices.

\subsection{Preliminaries}\label{sec:preliminaries}
We first introduce notation. Let $\simplex:=\{\bu\in[0,1]^{\nexperts}:\|\bu\|_{1}=1\}$ be the
$(\nexperts-1)$-dimensional probability simplex. Let $\alphasimplex :=
\{\bu\in[0,\alpha]^{\nexperts} : \|\bu\|_{1} = \alpha\}$ be a scaled
simplex.
Let $\bm{1}$ denote the vector $(1,\ldots,1)$ and $\bm{0}$ denote the
vector $(0,\ldots,0)$. Let $\bei$ denote the $i^{th}$ standard basis
vector.
We define $\RE{\bu}{\bw}
:=\sumin\ui\log{\frac{\ui}{\wi}}$ to be the relative entropy between
$\bu$ and $\bw$. We denote component-wise multiplication as
$\bu\odot\bw := (\ui[1]\wi[1],\ldots,\ui[n]\wi[n])$.
For $p\in[0,1]$ we define ${\cH(p) := -p\ln{p} - (1-p)\ln{(1-p)}}$ to be
the binary entropy of $p$, using the convention that ${0\ln{0}=0}$.
We define $\relint{S}$ to be the relative interior of the set $S$.
For any positive integer $n$ we define $[n] := \{1,\ldots,n\}$.
We overload notation such that $[\texttt{pred}]$ is equal to $1$ if the predicate
\texttt{pred} is true and $0$ otherwise.
For two vectors $\balpha$ and $\bbeta$ we say $\balpha\preceq\bbeta$
iff $\alphai\leq\betai$ for all $i=1,\ldots,\nexperts$.

\section{Background}\label{sec:background}
In sequential prediction with expert advice \nature\ generates
elements from an outcome space, $\cY$ while the predictions of \learner\
and the experts are elements from a prediction space, $\cD$ (e.g., we
may have $\cY=\{0,1\}$ and $\cD=[0,1]$). Given a
non-negative loss function  $\ell:\cD\times\cY\to [0,\infty)$,
learning proceeds in trials. On each trial $t=1,\ldots,T$: $1)$
\learner\ receives the expert predictions $\bxt\in\cD^{\nexperts}$,
$2)$ \learner\ makes a prediction $\ythat\in\cD$, $3)$ \nature\
reveals the true label $\yt\in\cY$, and $4)$ \learner\ suffers loss
$\lt:=\ell(\ythat,\yt)$ and expert $i$ suffers loss
$\lti:=\ell(\xti,\yt)$ for $i=1,\ldots,\nexperts$.
Common to the algorithms we consider in this paper is a weight vector,
$\bwt\in\simplex$, where $\wti$ can be interpreted as the algorithm's
confidence in expert $i$ on trial $t$.
\Learner\ uses a prediction function
$\pred:\simplex\times\cD^{\nexperts}\to\cD$ to generate its prediction
${\ythat = \predt}$ on trial $t$. A classic example is to predict with
the weighted average of the expert predictions, that is, $\predt
=\bwt\cdot\bxt$, although for some loss functions improved bounds are
obtained with different prediction functions (see
e.g.,~\cite{vovk1990aggregating}). In this paper we assume
$(c,\eta)$-realizability of $\ell$ and
\pred~\cite{bousquet2002tracking,haussler1998sequential,vovk1998game}. That
is, there exists constants $c,\eta>0$ such that for all
$\bw\in\simplex$, $\bx\in\cD^{\nexperts}$, and $y\in\cY$, 
${\ell(\textup{\pred}(\bw,\bx),
  y) \leq -c\ln{\sumin \vi e^{-\eta\ell(\xii,y)}}}$.
This includes $\eta$-exp-concave losses when $\predt =\bwt\cdot\bxt$
and $c=\frac{1}{\eta}$. For simplicity we present regret bound
guarantees that assume $(c,\frac{1}{c})$-realizability, that is
$c\eta=1$. This includes the log loss with $c=1$, and the square loss
with $c=\frac{1}{2}$. The absolute loss is {\em not}
$(c,\eta)$-realizable.  Generalizing our bounds for general bounded,
convex losses in the sense of online convex
optimization~\cite{zinkevich2003online} and the Hedge setting~\cite{freund1997decision} is straightforward. For any comparison sequence of experts
$\iit[1],\ldots,\iit[T]\in [\nexperts]$ the regret of \learner\ with
respect to this sequence is defined as 
\begin{equation*}
  \RT=\regret\,.
\end{equation*}
We consider and derive algorithms which belong to the family of
``exponential weights'' (EW) algorithms (see
e.g.,~\cite{vovk1990aggregating,hoeven2018many,littlestone1994weighted}).
After receiving the expert losses the EW algorithm applies the following
incremental loss update to the expert weights,  
\begin{equation}
  \label{eq:aa-update}
  \wtdoti =
  \frac{
    \wti e^{-\eta \lti}
  }{
    \sum_{j=1}^{\nexperts}\wtj e^{-\eta \lti[j]}
  }\,.
\end{equation}
\paragraph{Static setting.} In the static setting \learner\ competes against a
single expert (i.e., $\iit[1]=\ldots=\iit[T]$).
For the static setting the EW algorithm sets $\bwt[t+1]=\bwtdot$ for
the next trial, and for $(c,\frac{1}{c})$-realizable losses and
prediction functions achieves a static regret bound of $\RT\leq c\ln{\nexperts}$.

\paragraph{Switching.} In the switching (without memory) setting
\learner\ competes against a sequence of experts
${\iit[1],\ldots,\iit[T]}$ with
${k:=\sum_{t=1}^{T-1}[\iit\neq\iit[t+1]}]$ switches.
The well-known Fixed-Share algorithm~\cite{herbster1998tracking}
solves the switching problem with the update
\begin{equation}
  \label{eq:fs-update}
  \bwt[t+1] = (1-\alpha)\bwtdot + \alpha\frac{\bm{1}}{\nexperts}\,,
\end{equation}
by forcing each expert to ``share'' a fraction of its weight
\emph{uniformly} with all experts.\footnote{Technically in the original
  Fixed-Share update each expert shares weight to all \emph{other}
  experts, i.e., ${\wti[t+1]=(1-\alpha)\wtdoti+\frac{\alpha}{\nexperts-1}\sum_{j\neq
      i}\wtdotj}$. The two updates achieve essentially the same regret bound
  and are equivalent up to a scaling of $\alpha$.}
The update is parameterized by a ``switching'' parameter, 
$\alpha\in [0,1]$, and the regret with respect to the best sequence
of experts with $\nswitches$ switches is
\begin{equation}
  \label{eq:fs-regret-2}
  \RT\leq
  c\!\left(\!(\nsegments)\ln{\nexperts} +
    (T-1)\entropyBig{\frac{\nswitches}{T-1}}\!\right)\!
  \leq
  c\!\left(\!(\nsegments)\ln{\nexperts} +
    \nswitches\ln{\frac{T-1}{\nswitches}} +
    \nswitches\right).
\end{equation}
\paragraph{Switching with memory.}
Freund~\cite{Freund2000problem} gave an open problem to improve on
the regret bound~\eqref{eq:fs-regret-2} when the comparison sequence
of experts is comprised of a small pool of size
$\pool:=\vert\cup_{t=1}^{T}\{\iit\}\vert\ll\nswitches$.
Using counting arguments Freund gave an exponential-time algorithm
with the information-theoretic ideal regret bound of
${\RT\leq
  c\ln{(\binom{\nexperts}{\pool}\binom{T-1}{\nswitches}\pool(\pool-1)^{\nswitches})}}$,
which is upper-bounded by
\begin{equation}\label{eq:memory-target-bound}
  c\left(\pool\ln{\nexperts} +
    \nswitches\ln{\frac{T-1}{\nswitches}} +
    (\nswitches-\pool+1)\ln{\pool} +
    \nswitches+\pool\right)\,.
\end{equation}
The first efficient algorithm solving Freund's problem was presented
in the seminal paper~\cite{bousquet2002tracking}. This work
introduced the notion of a \emph{mixing scheme}, which is a
distribution $\bgammat$ with support $\{0,\ldots,t\}$. Given
$\bgammat$, the algorithm's update on each trial is  the
\emph{mixture} over all past weight vectors,
\begin{equation}
  \label{eq:mpp-mixture}
  \bwt[t+1]=\sum_{q=0}^{t}\gammatq\bwtdot[q]\,,
\end{equation}
where $\bwtdot[0]:=\frac{1}{\nexperts}\bm{1}$, and $\gamma^{1}_{0} := 1$. 
Intuitively, by mixing all ``past posteriors'' (MPP) the weights 
of previously well-performing experts can be prevented from
vanishing and recover quickly. An efficient mixing scheme requiring
$\cO(\nexperts)$-time per trial is the ``\emph{uniform}'' mixing
scheme given by $\gammatq[t]=1-\alpha$ and $\gammatq=\frac{\alpha}{t}$
for $0\leq q<t$. A better regret bound was proved with a ``\emph{decaying}''
mixing scheme, given by
\begin{equation}
  \label{eq:mpp-decay-scheme}
  \gammatq =
  \begin{cases}
    1-\alpha & q=t\\
    \alpha\frac{1}{(t-q)^{\gamma}}\frac{1}{Z_t} & 0\leq q<t\,,
  \end{cases}
\end{equation}
where $Z_t=\sum_{q=0}^{t-1}\frac{1}{(t-q)^{\gamma}}$ is a normalizing
factor, and $\gamma\geq 0$.
With a tuning of $\alpha=\frac{\nswitches}{T-1}$ and $\gamma=1$ this
mixing scheme achieves a regret bound of\footnote{\eqref{eq:mpp-decay-regret} is a simplified upper bound of the
  bound given in~\cite[Corollary 9]{bousquet2002tracking}, using
  $\ln{(1+x)}\leq x$.}
\begin{equation}
  \label{eq:mpp-decay-regret}
  \RT\leq c\left(
    \pool\ln{\nexperts} +
    2\nswitches\ln{\frac{T-1}{\nswitches}} +
    \nswitches\ln{(m-1)} +
    \nswitches +
    \nswitches\ln{\ln{(eT)}}
  \right)\,.
\end{equation}
It appeared that to achieve the best regret bounds, the mixing scheme 
needed to decay towards the past. 
Unfortunately, computing~\eqref{eq:mpp-decay-scheme} exactly requires
the storage of all past weights, at a cost of $\cO(\nexperts t)$-time
and space per trial. 
Observe that these schemes set $\gammatq[t]\!=\!1\!-\!\alpha$,
where typically $\alpha$ is small, since intuitively switches are
assumed to happen infrequently. All updates using such schemes are of
the form
\begin{equation}
  \label{eq:generalised-share-update}
  \bwt[t+1] = (1-\alpha)\bwtdot + \alpha\bvtdt\,,
\end{equation}
which we will call the {\em generalized share update} (see~\cite{cesa2012mirror}). Fixed-Share is a
special case when $\bvtdt=\frac{\bm{1}}{\nexperts}$ for all $t$. This
generalized share update features heavily in this paper.

For a decade it remained an open problem to give the MPP update a
Bayesian interpretation. This  was finally solved
in~\cite{koolen2012putting} with the use of \textit{partition
  specialists}. Here on each trial $t$, a specialist (first introduced
in~\cite{freund1997using}) is either \textit{awake} and predicts in
accordance with a prescribed base expert, or is \textit{asleep} and
abstains from predicting. For $\nexperts$ base experts and finite time
horizon $T$ there are $\nexperts 2^{T}$ partition specialists. For Freund's problem an assembly of $\pool$ partition specialists can
predict exactly as the comparison sequence of experts. The Bayesian
interpretation of the MPP update given in~\cite[Theorem
2]{koolen2012putting} was simple: to define a mixing scheme $\bgammat$
was to induce a prior over this set of partition specialists. 
The authors of~\cite{koolen2012putting} proposed a
simple Markov chain prior over the set of partition specialists,
giving an efficient $\cO(\nexperts)$-time per trial algorithm with the regret bound
\begin{align}
  \RT\!&\leq \!c\!\left[
         \pool\ln{\!\frac{\nexperts}{\pool}} \!+\!
         \pool\entropyBig{\frac{1}{\pool}} \!\!+\!
         (T-1)\entropyBig{\frac{\nswitches}{T-1}} \!\!+\!
         (\pool-1)(T-1)\entropyBig{\frac{\nswitches}{(\pool-1)(T-1)}}
         \!\!\right]\label{eq:pbts-regret}\\
       &\leq \!c\!\left(\!
         \pool\ln{\nexperts} \!+\!    
         2\nswitches\ln{\frac{T-1}{\nswitches}} \!+\!
         (\nswitches-\pool+1)\ln{\pool} \!+\!
         2(\nswitches+1)
         \!\right)\label{eq:pbts-regret-2}\,,
\end{align}
which is currently the best known regret bound for Freund's
problem. It is not known which MPP mixing scheme corresponds to this
Markov prior. In this work we improve on the
bound~\eqref{eq:pbts-regret} for tracking experts with memory
(Theorem~\ref{thm:pods-bound}), and also show that this Markov prior on
partition specialists corresponds to a geometrically-decaying mixing
scheme for MPP (Proposition~\ref{prop:pbts-as-mpp}).

Adaptive online learning algorithms with memory have been shown to
have better empirical performance than those without memory~\cite{gramacy2002adaptive},
and to be effective in real-world applications such as  intrusion detection
systems~\cite{nguyen2012adaptive}. While considerable research has
been done on switching with memory in online learning (see
e.g.,~\cite{bousquet2002tracking,cesa2012mirror,herbster2015predicting,
  herbster2020longterm,koolen2012putting,zheng2019equippingbandits}), 
there remain several open problems.
Firstly, there remains a gap between the
best known regret bound for an efficient algorithm and the
information-theoretic ideal bound~\eqref{eq:memory-target-bound}. 
Present in both bounds~\eqref{eq:mpp-decay-regret}
and~\eqref{eq:pbts-regret-2} is the factor of $2$ in the second term,
which does not appear
in~\eqref{eq:memory-target-bound}. In~\cite{koolen2012putting} this
was interpreted as the cost of co-ordination between specialists, essentially one ``pays'' twice per switch as one specialist falls 
asleep and another awakens.
In this paper we  make progress in closing
this gap by avoiding such additional costs the first time each expert is learned
by the  algorithm. That is, we pay to \emph{remember} but not to
\emph{learn}.  
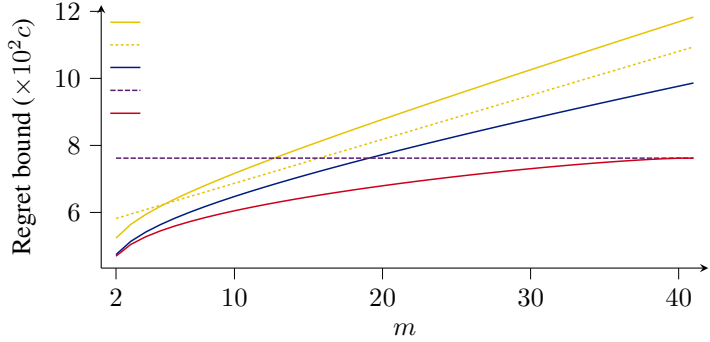
\begin{figure}[t]
	\centering
	\hspace*{-.6in}
	\pgfplotsset{
    compat=1.7,   
  legend image code/.code={
    \draw[mark repeat=2,mark phase=2]
    plot coordinates {
      (0cm,0cm)
      (0.3cm,0cm)        
      (0.4cm,0cm)         
    };
  }
}
\begin{tikzpicture} 

\definecolor{colorFixedShare}{rgb}{0.349019607843137,0.117647058823529,0.443137254901961}
\definecolor{colorPBTS}{rgb}{0,0.109803921568627,0.498039215686275}
\definecolor{colorMPP1}{rgb}{0.9, 0.76470588235, 0}
\definecolor{colorMPP2}{rgb}{0.23921569, 0.46666667, 0.00784314}
\definecolor{colorPODS}{rgb}{0.8, 0, 0.1}
\newcommand*\customlinewidth{0.57pt}

\begin{axis}[
height=2in,
legend cell align={left},
legend style={
  nodes={scale=0.7, transform shape},
  fill opacity=0,
  draw opacity=1,
  text opacity=1,
  at={(0,1)},
  anchor=north west,
  draw=none
},
axis lines=left,
tick align=outside,
tick pos=left,
width=3.8in,
x grid style={white!69.0196078431373!black},
xlabel={\(\displaystyle m\)},
xmin=1, xmax=42,
xtick style={color=black},
xtick={2,10,20,30,40},
xticklabels={
  \(\displaystyle {2}\),
  \(\displaystyle {10}\),
  \(\displaystyle {20}\),
  \(\displaystyle {30}\),
  \(\displaystyle {40}\)
},
y grid style={white!69.0196078431373!black},
ylabel={Regret bound ($\times 10^{2} c$)},
ymin=4.3395347550838, ymax=12.186059658039,
ytick style={color=black}
]

\addplot [line width=\customlinewidth,colorMPP1] 
table {%
2 5.23612594604492
3 5.64460849761963
4 5.93801832199097
5 6.1843147277832
6 6.40479564666748
7 6.60894775390625
8 6.80183172225952
9 6.98646783828735
10 7.16480493545532
11 7.33817291259766
12 7.50752019882202
13 7.67354869842529
14 7.83678913116455
15 7.99765634536743
16 8.1564769744873
17 8.31351566314697
18 8.46898937225342
19 8.62307643890381
20 8.77592658996582
21 8.92766761779785
22 9.07840728759766
24 9.3772439956665
26 9.6730432510376
28 9.96627521514893
30 10.257306098938
32 10.5464296340942
34 10.8338851928711
36 11.1198682785034
38 11.4045438766479
41 11.8293991088867
};
\addlegendentry{MPP (decaying scheme)}

\addplot [line width=\customlinewidth,colorMPP1,dash pattern=on 1pt
off 1pt] 
table {%
2 5.81992816925049
41 10.9376497268677
};
\addlegendentry{MPP (uniform scheme)}

\addplot [line width=\customlinewidth,colorPBTS] 
table {%
2 4.74236917495728
3 5.13799381256104
4 5.41264247894287
5 5.63661241531372
6 5.8321738243103
7 6.00935888290405
8 6.17358255386353
9 6.32811260223389
10 6.47508239746094
11 6.61596441268921
12 6.75182008743286
13 6.8834433555603
14 7.01144218444824
15 7.13629531860352
16 7.25838613510132
17 7.37802696228027
18 7.49547576904297
19 7.61094856262207
20 7.72462844848633
21 7.83667182922363
22 7.94721269607544
23 8.05636882781982
24 8.1642427444458
25 8.27092456817627
26 8.37649440765381
27 8.48102378845215
28 8.58457565307617
30 8.78897190093994
32 8.99007987976074
34 9.18822765350342
36 9.38368988037109
38 9.57670021057129
40 9.76745891571045
41 9.86204814910889
};
\addlegendentry{Partition Specialists (Markov prior)}

\addplot [line width=\customlinewidth, colorFixedShare, dash pattern=on 2pt off .75pt]
table {%
2 7.62013006210327
41 7.62013006210327
};
\addlegendentry{Fixed-Share}

\addplot [line width=\customlinewidth,colorPODS]
table {%
2 4.69619512557983
3 5.04538917541504
4 5.27334451675415
5 5.45035028457642
6 5.59867000579834
7 5.72832775115967
8 5.84472942352295
9 5.95113515853882
10 6.04966735839844
11 6.14178943634033
12 6.22855186462402
13 6.31073665618896
14 6.38894033432007
15 6.46362781524658
16 6.5351676940918
17 6.60385799407959
18 6.6699390411377
19 6.73360967636108
20 6.7950325012207
21 6.85434198379517
22 6.91164970397949
23 6.96704578399658
24 7.0206036567688
25 7.07238054275513
26 7.12242031097412
27 7.17075204849243
28 7.21739149093628
29 7.26234149932861
30 7.30558919906616
31 7.34710550308228
32 7.38684177398682
33 7.4247260093689
34 7.46065282821655
35 7.49447584152222
36 7.5259838104248
37 7.55486583709717
38 7.58063411712646
39 7.60244226455688
40 7.61847734451294
41 7.62013006210327
};
\addlegendentry{\podsth\ / \shareth}

\end{axis}

\end{tikzpicture}

	\caption{A comparison of the regret bounds discussed in this paper
		for ${\pool\!\in\![2,\nswitches\!+\!1]}$ with ${n\!=\!500000}$, ${\nswitches\!=\!40}$,
		and ${T\!=\!4000}$. Previous ``memory'' bounds (blue \& yellow)
		are much worse than Fixed-Share for larger values of $\pool$ while
		our bound (red) improves on Fixed-Share for all $\pool\!\in\![2,\nswitches]$.}
	\label{fig:regret-bounds}
\end{figure}

Secondly, unless $\nexperts$ is very large
the current best known bound~\eqref{eq:pbts-regret} beats Fixed-Share's
bound~\eqref{eq:fs-regret-2} only when $\pool\ll\nswitches$,
but suffers when $m$ is even a moderate fraction of
$k$.
A natural question is can we improve on Fixed-Share when we relax the
assumption that $\pool\ll\nswitches$, and only a few members of a
sequence of experts need remembering (consider for instance, $\pool >
\nswitches/2$)? 
In this paper we prove a regret bound that is not only tighter
than~\eqref{eq:pbts-regret} for all $\pool$, but under mild
assumptions on $\nexperts$ improves on Fixed-Share for all $m\leq
k$. See Figure~\ref{fig:regret-bounds} where we show this behavior for
several existing regret bounds and our regret bound.

Our regret bound will hold for two algorithms; one utilizes a weight-sharing
update in the sense of~\eqref{eq:generalised-share-update}, and the other
utilizes a projection update. Why should we consider projections? Consider for example a large
model consisting of many weights, and to update these weights costs time and/or money. Alternatively consider the
application of regret-bounded adaptive algorithms in online portfolio
selection (see e.g., \cite{singer1997switching,li2014online}). Here
each ``expert'' corresponds to a single asset, and the weight vector
$\bwt$ corresponds to a portfolio. If $\lti$ is the negative log
return of stock $i$ after trading period $t$, then the loss
function ${\lt:=-\ln{\sumin\wti e^{-\lti}}}$ is the negative log return
of the portfolio. This loss is $(1,1)$-realizable by 
definition (although there is no prediction function~\cite{adamskiy2016closer}). The algorithm's
update corresponds to actively re-balancing the portfolio
after each trading period, but the investor may incur
transaction costs proportional to the amount bought or sold (see
e.g.,~\cite{blum1999universal,li2014online}). Online portfolio
selection with transaction costs is an active area of
research~\cite{das2013online,kozat2008universal,li2014online,li2018transaction}. In
Section~\ref{sec:pods-update} we motivate the use of
projection updates over weight-sharing with a guarantee in terms
of such costs.

\subsection{Related work}\label{sec:related-work}
Switching (without memory) in online learning was first introduced
in~\cite{littlestone1994weighted}, and extended with the Fixed-Share
algorithm~\cite{herbster1998tracking}. An extensive literature has
built on these works, including but not limited
to~\cite{adamskiy2016closer,bousquet2002tracking,cesa2012mirror,daniely2015strongly,gyorgy2005tracking,gyorgy2012efficient,herbster2020longterm,herbster2019online,herbster2001tracking,koolen2012putting,koolenl2010freezingsleeping,mourtada2017efficient,sharma2020learning,zheng2019equippingbandits}. 
Relevant to this work are the results for switching with
memory~\cite{bousquet2002tracking,cesa2012mirror,herbster2020longterm,koolen2012putting,koolenl2010freezingsleeping,zheng2019equippingbandits}. The
first was the seminal work of~\cite{bousquet2002tracking}. The best known result
is given in~\cite{koolen2012putting}, which we improve on.
In~\cite{zheng2019equippingbandits} a reduction of switching
with memory to switching without memory is given, although with a slightly
worse regret bound than~\cite{bousquet2002tracking}. Related to
the experts model is the {\em bandits} setting, which was addressed in
the memory setting in~\cite{zheng2019equippingbandits}.
In~\cite{cesa2012mirror} a unified analysis of both Fixed-Share and
MPP was given in the context of online convex optimization. They observed the
generalized share update~\eqref{eq:generalised-share-update} and
slightly improved the bounds of~\cite{bousquet2002tracking}.
Adaptive
regret~\cite{adamskiy2016closer,daniely2015strongly,hazan2009efficient,littlestone1994weighted} has been used to prove regret bounds for
switching but unfortunately does not generalize to the memory setting.
This paper primarily builds on the work of~\cite{bousquet2002tracking}
with a new geometrically-decaying mixing scheme, and
on~\cite{herbster2001tracking} with a new relative entropy projection algorithm.

\section{Projection onto dynamic sets}\label{sec:pods}
In this section we give a relative entropy projection-based algorithm
for tracking experts with memory.
Given a non-empty set $\cC\subseteq\simplex$ and a point
$\bw\in\relint{\simplex}$ we define
\begin{equation*}
  \projw[\cC] := \argmin\limits_{\bu\in\cC}\RE{\bu}{\bw}\, 
\end{equation*}
to be the projection with respect to the relative entropy of $\bw$
onto $\cC$~\cite{bregman1967relaxation}.
Such projections were first introduced for switching (without
memory) in online learning in~\cite{herbster2001tracking}, in which
after every trial the weight vector $\bwtdot$ is projected onto 
$\cC=[\frac{\alpha}{\nexperts},1]^{\nexperts}\cap\simplex$, that is,
the simplex with uniform box constraints.
For prediction with expert advice this projection algorithm has the
regret bound~\eqref{eq:fs-regret-2} (see~\cite{cesa2012mirror}).  
Indeed, we will refer to 
$\bwt[t+1]=\proj[\bwtdot]{[\frac{\alpha}{\nexperts},1]^{\nexperts}\cap\simplex}$
as the ``projection analogue'' of~\eqref{eq:fs-update}.  

Given $\bbeta\in (0,1)^{\nexperts}$ such that $\|\bbeta\|_{1}\leq 1$, let  
\begin{equation*}
  \Cb:=\{\bx\in\simplex : \xii\geq\betai, i=1,\ldots,\nexperts\}
\end{equation*}
be a subset of the simplex which is convex and non-empty.
Given $\bw\in\relint{\simplex}$, intuitively $\projw[\Cb]$ is  the projection of $\bw$ onto
the simplex with (non-uniform) lower box constraints $\bbeta$. 
Relative entropy projection updates for tracking experts with memory
were first suggested in~\cite[Section 5.2]{bousquet2002tracking}. The
authors observed that for any MPP mixing scheme $\bgammat$, the update~\eqref{eq:mpp-mixture} can be replaced with
\begin{equation}
  \label{eq:mpp-projection}
  \bwt[t+1]=\proj[\bwtdot]{\{\bw\in\simplex: \bw\succeq\gammatq\bwqdot,q=0,\ldots,t\}}\,,
\end{equation}
and achieve the same regret bound.
We build on this concept in this paper. 
Observe that for any choice of $\bgammat$ the set $\{\bw\in\simplex:
\bw\succeq\gammatq\bwqdot,q=0,\ldots,t\}$ corresponds to the set $\Cb$
where
\begin{equation}
  \label{eq:beta-def}
  \betai = \max\limits_{0\leq q\leq
    t}\gammatq\wtdoti[q]\qquad i=1,\ldots,\nexperts\,.
\end{equation}
In this work we give an algorithm to compute $\projw$ exactly
for any $\Cb$ in $\cO(\nexperts)$ time. With this algorithm
and the mapping~\eqref{eq:beta-def}, one immediately obtains the
projection analogue of MPP for any mixing scheme $\bgammat$ at
essentially no additional computational cost.
We point out however that for arbitrary mixing schemes computing
$\bbeta$ from~\eqref{eq:beta-def} takes $\cO(\nexperts t)$-time on
trial $t$, improving only when some structure of the scheme can be
exploited.
We therefore propose the following method for tracking
experts with memory {\em efficiently} using projection onto dynamic
sets (``PoDS''). 

Just as~\eqref{eq:generalised-share-update} generalizes the Fixed-Share update~\eqref{eq:fs-update}, we
propose PoDS as the analogous generalization of the update $\bwt[t+1] \!=\!
\proj[\bwtdot]{\Cb[\alpha\frac{\bm{1}}{\nexperts}]}$ (the projection analogue of
Fixed-Share). 
PoDS maintains a vector $\bbetat\in\alphasimplex$, and on each trial
updates the weights by setting ${\bwt[t+1] \!=\! \proj[\bwtdot]{\Cbt}}$.
Intuitively PoDS is the projection analogue
of~\eqref{eq:generalised-share-update} with $\bbetat$ corresponding
simply to $\alpha\bvtdt$. In some cases
$\bbetat\!=\!\alpha\bvtdt$ for all $t$ (e.g., for Fixed-Share), but
in general equality may not hold since $\bbetat$
and $\bvtdt$ can be functions of past weights, which may differ for
weight-sharing and projection algorithms.
Recall that~\eqref{eq:generalised-share-update}
describes all MPP mixing schemes that set $\gammatq[t]\!=\!1\!-\!\alpha$. PoDS
implicitly captures all such mixing schemes. This simple
formulation of PoDS allows us to define new updates, which will
correspond to new mixing schemes. In Section~\ref{sec:pods-update} we
give a simple update and prove the best known regret bound.

\subsection{Computing
  \texorpdfstring{$\projw$}{P(w;C(B))}}\label{sec:computing-projection}
Before we consider PoDS further, we first discuss
the computation of $\projw$.
In~\cite{herbster2001tracking} the authors showed that computing
relative entropy projection onto the simplex with {\em uniform} box
constraints is non-trivial, but gave an algorithm to compute it in
$\cO(\nexperts)$ time. We give a generalization of their algorithm
to compute $\projw$ exactly for any non-empty set $\Cb$ in $\cO(\nexperts)$ time.
As far as we are aware our method to compute exact relative entropy
projection onto the simplex with non-uniform (lower) box constraints
in linear time is the first, and may be of independent interest (see e.g.,~\cite{krichene2015efficient}).

We first give an intuition into the form of $\projw[\Cb]$, and then
describe how Algorithm~\ref{alg:linear-time-proj} computes this projection efficiently. 
Firstly consider the case that $\bw\in\Cb$, then trivially $\projw[\Cb]=\bw$,
due to the non-negativity of $\RE{\bu}{\bw}$
and the fact that $\RE{\bu}{\bw}=0$ iff $\bu=\bw$ (see
e.g.,~\cite{bregman1967relaxation}). For the case that $\bw\notin\Cb$,
this implies that the set $\{i\in [n] : \wi<\betai\}$ is
non-empty. For each index $i$ in this set the projection of $\bw$ onto
$\Cb$ must set the component $\wi$ to its corresponding constraint
value $\betai$. The remaining components are then normalized, such that
$\sumin \wi=1$. However, doing so may cause one (or more) of these
components $\wi[j]$ to drop below its constraint $\betai[j]$. The projection algorithm therefore finds the set of
components $\Psi$ of least cardinality to set to their constraint
values such that when the remaining components are normalized, no component lies below its constraint.

Consider the following inefficient approach to finding $\Psi$. Given
$\bw$ and $\Cb$, let $\br=\bw\odot\frac{1}{\bbeta}$ be a 
``ratio vector''.  Then sort $\br$ 
in ascending order, and sort $\bw$ and $\bbeta$ according
to the ordering of $\br$. 
If $\ri[1] \geq 1$ then $\Psi=\emptyset$ and we are done
($\Rightarrow\bw\in\Cb$). 
Otherwise for each $k=1,\ldots,\nexperts$: $1)$ let the candidate set
$\Psi^{'} = [k]$, $2)$ let $\bw'=\bw$ except for each $i\in\Psi^{'}$
set $\wi'=\betai$, $3)$ re-normalize the remaining components of
$\bw'$, and $4)$ let $\br'=\bw'\odot\frac{1}{\bbeta}$. The set $\Psi$
is then the candidate set $\Psi^{'}$ of least cardinality such that
$\br' \succeq \bm{1}$. This approach requires sorting $\br$
and therefore even an efficient implementation
takes $\cO(\nexperts \log{\nexperts})$
time. Algorithm~\ref{alg:linear-time-proj} finds $\Psi$ without having
to sort $\br$. It instead specifies $\Psi$ uniquely with a threshold,
$\threshold$, such that $\Psi=\{i : \ri < \threshold\}$.
Algorithm~\ref{alg:linear-time-proj} finds $\threshold$
through repeatedly bisecting the set $\mW=[\nexperts]$ by finding the
median of the set $\{\ri : i \in \mW\}$ (which can be done in
$\cO(\vert\mW\vert)$ time~\cite{blum1973time}), and efficiently
testing this value as the candidate threshold on each iteration. The smallest
valid threshold then specifies the set $\Psi$.
The following theorem states the time complexity of the algorithm and the form of the projection, which
we will use in proving our regret bound (the proof is in
Appendix~\ref{sec:proj-linear-time-proof}, where we give a more
detailed description of the algorithm).
\begin{theorem}\label{thm:projection-form}
  For any $\bbeta\in (0,1)^{\nexperts}$ such that
  $\onenorm{\bbeta}\leq 1$, and for any $\bw\in\relint{\simplex}$, let
  ${\bp=\projw[\Cb]}$, where $\Cb=\{\bx\in\simplex : \xii\geq\betai, i=1,\ldots,\nexperts\}$. Then
  $\bp$ is such that for all $i=1,\ldots,n$, 
  \begin{equation}
    \label{eq:proj-max-form}
    \pii = \max\left\{
      \betai;
      \frac{
        1-\sum_{j\in\Psi}\betai[j]
      }{
        1 - \sum_{j\in\Psi}\wi[j]
      }\wi
    \right\}\,,
  \end{equation}
  where $\Psi := \{i \in [n] : \pii = \betai\}$. Furthermore,
  Algorithm~\ref{alg:linear-time-proj} computes $\bp$ in
  $\cO(\nexperts)$ time.
\end{theorem}
The following corollary will be used in the proof of our regret bound.
\begin{corollary}\label{cor:no-switch}
  Let $0<\alpha<1$. Then for any $\bu\in\simplex$,
  $\bw\in\relint{\simplex}$, and $\bbeta\in\relint{\alphasimplex}$,
  let $\bp=\proj[\bw]{\Cb}$. Then,
  \begin{equation}\label{eq:pods-no-switch}
    \RE{\bu}{\bw} - \RE{\bu}{\bp} \geq
    \ln{(1-\alpha)}\,. 
  \end{equation}  
\end{corollary}

\subsection{A simple update rule for PoDS}\label{sec:pods-update}
We now suggest a simple update rule for $\bbetat$ in PoDS for tracking
experts with memory. The bound for this algorithm is given in Theorem~\ref{thm:pods-bound}. 
We first set $\bbetat[1]=\alpha\frac{\bm{1}}{\nexperts}$ to be uniform, and with a
parameter $0\leq\theta\leq 1$ update $\bbetat$ on subsequent trials by setting
\begin{equation}
  \label{eq:update-beta}
  \bbetat[t+1] = (1-\theta)\bbetat + \theta\alpha\bwtdot\,.
\end{equation}
We refer to PoDS with this
update as \podsth. Intuitively the constraint
vector $\bbetat$ is updated in~\eqref{eq:update-beta} by mixing in a
small amount of the current  weight vector, $\bwtdot$, scaled
such that $\onenorm{\bbetat[t+1]}=\alpha$. If expert $i$ predicted well
in the past, then its constraint $\betati$ will be relatively large,
preventing the weight from vanishing even if that expert suffers large
losses locally. Using Algorithm~\ref{alg:linear-time-proj} in its
projection step, \podsth\ has  $\cO(\nexperts)$ per-trial time complexity. 
\begin{figure*}[t]
  \begin{minipage}[t]{0.48\linewidth}
    \begin{algorithm}[H]
      \floatname{algorithm}{Algorithms}
      \caption{\podsth\ / \shareth\ \phantom{$($}}
      \label{alg:pods}
      \algoInput{$\nexperts>0$, $\eta=\frac{1}{c}>0$, $\alpha\in
        [0,1]$, $\theta\in [0,1]$}
      \begin{algorithmic}[1]
        \algrule \vspace{-.1\baselineskip}
        \algotitle{PoDS-$\theta$}
        \Init{$\bwt[1]\GETS\frac{\bm{1}}{\nexperts}$;
          $\bbetat[1]\GETS\alpha\frac{\bm{1}}{\nexperts}$}
        \algrule
        \setcounter{ALG@line}{0}
        \algotitle{\shareth}
        \Init{$\bwt[1]\GETS\frac{\bm{1}}{\nexperts}$;
          $\bvt[1]\GETS\frac{\bm{1}}{\nexperts}$}
        \algrule
        \algotitle{PoDS-$\theta$ \& \shareth}        
        \For{$t \GETS 1 \textrm{ to } T$}
        \Receive{$\bxt\in\cD^{\nexperts}$}
        \Predict{$\ythat=\predt$}
        \Receive{$\yt\in\cY$}
        \For{$i \GETS 1 \textrm{ to } \nexperts$}
        \State $\wtdoti \GETS \frac{\wti
          e^{-\eta\lti}}{\sum_{j=1}^{\nexperts}\wtj e^{-\eta\lti[j]}}$
        \EndFor
        \algrule
        \algotitle{PoDS-$\theta$}
        \State $\bwt[t+1] \GETS
        \proj[\bwtdot]{\Cbt}\hfill\refstepcounter{equation}(\theequation)\label{eq:update-proj}$
        \State $\bbetat[t+1] \GETS (1-\theta)\bbetat[t] +
        \theta\alpha\bwtdot$

        \algrule
        \setcounter{ALG@line}{7}
        \algotitle{\shareth}
        \State $\bwt[t+1] \GETS (1-\alpha)\bwtdot
        +\alpha\bvt\hfill\refstepcounter{equation}(\theequation)\label{eq:update-wt}  $
        \State $\bvt[t+1] \GETS (1-\theta)\bvt +
        \theta\bwtdot\hfill\refstepcounter{equation}(\theequation)\label{eq:update-vt}$
        \EndFor
      \end{algorithmic}
    \end{algorithm}
  \end{minipage}
  \hfill
  \begin{minipage}[t][][b]{0.48\linewidth}
    \renewcommand{\thealgorithm}{\arabic{algorithm}}
    \setcounter{algorithm}{2}
    \begin{algorithm}[H]
      \caption{$\proj[\bw]{\Cb}$ in $\cO(\nexperts)$ time}
      \label{alg:linear-time-proj}
      \algoInput{$\bw\in\relint{\simplex};
        \bbeta\in(0,1)^{n} \text{ s.t. } \onenorm{\bbeta}\leq 1$}\\
      \algoOutput{$\bw'=\proj[\bw]{\Cb}$}
      \begin{algorithmic}[1]
        \Init{$\mW\GETS [\nexperts]$;
          $\br\GETS\bw\odot\frac{1}{\bbeta}$; $S_{\bw}\GETS 0$;
          $S_{\bbeta}\GETS 0$}
        \While{$\mW \neq \emptyset$}\label{al:whilestart}
        \State \label{al:median}$\threshold \GETS \text{\tt median}(\{ \ri : i \in \mW\})$
        \State $\mL \GETS \{i \in \mW : \ri < \threshold\}$
        \State $L_{\bbeta} \GETS \sum_{i\in\mL}\beta_{i};\,  L_{\bw} \GETS \sum_{i\in\mL}\wi$
        \State $\mM \GETS \{i \in \mW : \ri = \threshold\}$
        \State $M_{\bbeta} \GETS \sum_{i\in\mM}\beta_{i};\,  M_{\bw} \GETS \sum_{i\in\mM}\wi$
        \State $\mH \GETS \{i \in \mW : \ri > \threshold\}$
        \State \label{al:lambda}$\lambda \GETS \frac{1 - S_{\bbeta} - L_{\bbeta}}{1 - S_{\bw} - L_{\bw}}$
        \If{$\threshold\lambda < 1$}
        \State $S_{\bw} \GETS S_{\bw} +  L_{\bw} +  M_{\bw}$
        \State $S_{\bbeta} \GETS S_{\bbeta} +  L_{\bbeta} +  M_{\bbeta}$
        \If{$\mH = \emptyset$}
        \State $\threshold \GETS \text{\tt min}(\{\ri : \ri >
        \threshold, i \in [\nexperts]\})$
        \EndIf
        \State $\mW \GETS \mH$
        \Else 
        \State\label{al:whileend} $\mW \GETS \mL$
        \EndIf
        \EndWhile
        \State $\lambda \GETS \frac{1 - S_{\bbeta}}{1-S_{\bw}}$
        \vspace{-0.85\baselineskip}
        \State $\forall i : 1,\ldots,n : w'_{i} \GETS
        \begin{cases}
          \beta_{i} &\quad \ri <\threshold \\
          \lambda \wi &\quad \ri \geq \threshold
        \end{cases}$
      \end{algorithmic}\vspace{-0.4em}
    \end{algorithm}
  \end{minipage}
\end{figure*}

As discussed, the vector $\bbetat$ of PoDS is
conceptually equivalent to the vector $\alpha\bvtdt$ of
the generalized share update~\eqref{eq:generalised-share-update}.
If PoDS has a simple update rule such as~\eqref{eq:update-beta} then it is
straightforward to recover the weight-sharing equivalent by simply 
``pretending'' equality holds on all trials. We now do this for \podsth.
Clearly we have $\bvtdt[1]=\frac{\bm{1}}{\nexperts}$, and if $\bbetat=\alpha\bvtdt$
and $\bbetat[t+1]=\alpha\bvtdt[t+1]$, then
${\bvtdt[t+1]\!=\!\frac{1}{\alpha}\bbetat[t+1]\!=\!\frac{1}{\alpha}(1\!-\!\theta)\bbetat
+ \theta\bwtdot\!=\!(1\!-\!\theta)\bvtdt+ \theta\bwtdot}$.
This then leads to an efficient sharing algorithm, which we call
\shareth. In Section~\ref{sec:mpp-mixing-scheme} we show this
algorithm is in fact a new MPP mixing scheme, which surprisingly
corresponds to the previous best known algorithm for this problem.
Both \podsth\ and \shareth\ use the same parameters ($\alpha$ 
and $\theta$), differing only in the final update (see
Algorithms~\ref{alg:pods}). We now give the regret bound which holds
for both algorithms.
\begin{theorem}\label{thm:pods-bound}
  For any comparison sequence $\iit[1],\ldots,\iit[T]$ containing
  $\nswitches$ switches and consisting of $\pool$ unique
  experts from a set of size $\nexperts$, if $\alpha =
  \frac{\nswitches}{T-1}$ and $\theta=\frac{\nswitches-\pool+1}{(\pool
    -1)(T-2)}$, the regret of both \podsth\ and \shareth\, with any prediction
  function and loss function which are $(c,\frac{1}{c})$-realizable is   
  \begin{equation}
    \label{eq:pods-bound}
    \RT\leq
    c\!\left(\!
      \pool\ln{\nexperts} +\!
      (T-1)\entropyBig{\frac{\nswitches}{T-1}}\! +
      (\pool -1)(T-2)\entropyBig{
        \frac{
          \nswitches-\pool+1
        }{
          (\pool -1)(T-2)
        }
      }\!
    \right).
  \end{equation}
\end{theorem}
The regret bound~\eqref{eq:pods-bound} is at least
$c((\pool\!-\!1)\ln{\frac{T-1}{\nswitches}}\!-\!(\nswitches\!-\!\pool\!+\!1)\ln{\frac{\nswitches}{\nswitches-\pool+1})})$
tighter than the currently best known
bound~\eqref{eq:pbts-regret}. Thus if $\pool\!\ll\!\nswitches$
then the improvement is $\approx\!
c\pool\ln{\frac{T}{\nswitches}}$, and as $\pool\!\rightarrow\!\nswitches\!+\!1$
then the improvement is $\approx\!
c\nswitches\ln{\frac{T}{\nswitches}}$.  
Additionally note that if $\pool\!=\nswitches\!+\!1$ (i.e., every switch we
track a \textit{new} expert) the optimal tuning of $\theta$ is zero,
and \podsth\ reduces to setting
$\bbetat=\alpha\frac{\bm{1}}{\nexperts}$ on
every trial. That is, we recover the projection analogue of
Fixed-Share. This is also reflected in the
regret bound since~\eqref{eq:pods-bound} reduces to~\eqref{eq:fs-regret-2}.
Since $x\cH(\frac{y}{x})\leq y\ln{(\frac{x}{y})} + y$, the regret
bound~\eqref{eq:pods-bound} is upper-bounded by
\begin{equation*}
  \RT\leq\!
  c\!\left[\!
    m\ln{\nexperts} \!+\!
    \nswitches\ln{\frac{T\!-\!1}{\nswitches}} \!+\!
    (\nswitches\!-\!\pool\!+\!1)\ln{\frac{T\!-\!2}{\nswitches\!-\!\pool\!+\!1}} \!+\!
    (\nswitches\!-\!\pool\!+\!1)\ln{(\pool\!-\!1)} \!+\!
    2\nswitches-\pool + 1\right].
\end{equation*}
Comparing this to~\eqref{eq:pbts-regret-2}, we see that instead
of paying $c\ln{\frac{T-1}{\nswitches}}$ {\em twice} on every switch, we pay
$c\ln{\frac{T-1}{\nswitches}}$ once per switch and $c\ln{\frac{T-2}{\nswitches-\pool+1}}$
for every switch we {\em remember} an old expert
(${\nswitches-\pool+1}$ times).
Unlike previous results for tracking experts with memory, \podsth\ and its regret
bound~\eqref{eq:pods-bound} smoothly interpolate between the two
switching settings. That is, it is capable of exploiting memory when necessary
and on the other hand does not suffer when memory is not necessary
(see Figure~\ref{fig:regret-bounds}). 

\paragraph{Projection vs. sharing in online learning.}
We now briefly consider the two types of updates discussed in this
paper (projection and weight-sharing) when updating weights may incur costs.
Recall the motivating example introduced in
Section~\ref{sec:background} was in online portfolio selection with
transaction costs.
It is straightforward to show that in
this model transaction costs are proportional to the $1$-norm of the
difference in the weight vectors before and after re-balancing. 
In Theorem~\ref{thm:proj-lt-sharing-general} we give a
 result which in this context guarantees the ``cost'' of projecting is
 less than that of weight-sharing.

To compare the update of PoDS and the generalized share
update~\eqref{eq:generalised-share-update}, we must consider for a set
of weights $\bwtdot$, the point $\proj[\bwtdot]{\Cbt}$ and the point
${(1-\alpha)\bwtdot + \alpha\bvtdt}$. However these points depend on
$\bbetat$ and $\bvtdt$ respectively, which may themselves be functions of
previous weight vectors $\bwtdot[1],\ldots,\bwtdot[t-1]$, which as
discussed are generally not the same for each of the two algorithms.
To compare the two updates equally we therefore assume that the
current weights are the same (i.e., they must both update the same
weights $\bwtdot$), and additionally that $\bbetat=\alpha\bvtdt$.
The following theorem states that under mild conditions, PoDS is strictly less ``expensive'' than its weight-sharing
counterpart.
\begin{theorem}\label{thm:proj-lt-sharing-general}
  Let $0<\alpha<1$. Then for any $\bv\in\relint{\simplex}$,
  let $\bbeta=\alpha\bv$, and for any $\bw\in\relint{\simplex}$, let
  ${\bw'=(1-\alpha)\bw + \alpha\bv}$. Then,
  \begin{equation*}
    \onenormBig{\projw[\Cb]-\bw} < \onenormBig{\bw'-\bw}\,.
  \end{equation*}
\end{theorem}
Thus if one has to pay to update weights, projection is the economical choice. 

\section{A geometrically-decaying mixing scheme for
  MPP}\label{sec:mpp-mixing-scheme}
In this section we look more closely at  \shareth. We
show that it is in fact a new type of {\em decaying} MPP 
mixing scheme which corresponds to the partition specialist algorithm with
Markov prior.

Recall that the previous best known mixing scheme for MPP is the
decaying scheme~\eqref{eq:mpp-decay-scheme}. Observe that
in~\eqref{eq:mpp-decay-scheme} the decay (with the ``distance'' to the
current trial $t$) follows a power-law, and that
computing~\eqref{eq:mpp-decay-scheme} exactly takes $\cO(\nexperts t)$
time per trial. We now derive an explicit MPP mixing scheme from the 
updates~\eqref{eq:update-wt} and~\eqref{eq:update-vt} of \shareth. Observe that if
we define $\bwtdot[0]:=\frac{\bm{1}}{\nexperts}$, then an iterative
expansion of~\eqref{eq:update-vt} on any trial $t$ gives
${\bvt=\sum_{q=0}^{t-1}\theta^{[q\neq
    0]}(1-\theta)^{t-q-1}\bwtdot[q]}$,
from which~\eqref{eq:update-wt} implies
${\bwt[t+1] = (1-\alpha)\bwtdot + \alpha\bvt =
  \sum_{q=0}^{t}\gammatq\bwtdot[q]}$,
where 
\begin{equation}\label{eq:new-mixing-scheme}
  \gammatq =
  \begin{cases}
    1-\alpha & q=t\\
    \theta(1-\theta)^{t-q-1}\alpha & 1\leq q < t\\
    (1-\theta)^{t-1}\alpha & q=0\,.
  \end{cases}\,.
\end{equation}
Note that~\eqref{eq:new-mixing-scheme} is a valid mixing scheme since
for all $t$, $\sum_{q=0}^{t}\gammatq = 1$. The \shareth\ update is 
therefore a new kind of decaying mixing scheme. 
In this new scheme the decay is {\em geometric}, and can therefore be
computed efficiently, requiring only $\cO(\nexperts)$ time and space
per trial as we have shown. Furthermore MPP with this scheme has the improved regret
bound~\eqref{eq:pods-bound}.

Another interesting difference between the decaying
schemes~\eqref{eq:new-mixing-scheme} and~\eqref{eq:mpp-decay-scheme} is
that when $\theta$ is small then~\eqref{eq:new-mixing-scheme} keeps
$\gammatq[0]$ relatively large initially and slowly decays this value as
$t$ increases. Intuitively by heavily weighting the initial uniform
vector $\bwtdot[0]$ on each trial early on, the algorithm
can ``pick up'' the weights of new experts easily.
Finally as in the case of \podsth, if $\pool=\nsegments$, then with the optimal
tuning of $\theta=0$, this update reduces to the Fixed-Share update~\eqref{eq:fs-update}.

\paragraph{Revisiting partition specialists.}
We now turn our attention to the previous best known result for
tracking experts with memory (the partition specialists
algorithm with a Markov prior~\cite{koolen2012putting}).

For sleep/wake patterns $(\chit[1]\ldots\chit[T])$ the Markov prior
is a Markov chain on states $\{w,s\}$, defined by the initial
distribution $\bpi\!=\!(\piw, \pis)$ and transition probabilities
$\Pij\!:=\!P(\chit[t+1]\!=\!j\vert\chit\!=\!i)$ for $i,j\in \{w,s\}$.
The algorithm with these inputs efficiently collapses one weight per
specialist down to two weights per expert. These two weight vectors, which we denote
$\bat$ and $\bst$, represent the total weight of all awake and sleeping specialists
associated with each expert, respectively. Note that the vectors
$\bat$ and $\bst$ are not in the  simplex, but rather the vector
$(\bat,\bst)\in\Delta_{2\nexperts}$ and the ``awake vector'' $\bat$ gets
normalized upon prediction. The weights are initialized by
setting $\bat[1] = \piw\frac{\bm{1}}{\nexperts}$, and $\bst[1]
=\pis\frac{\bm{1}}{\nexperts}$. The update\footnote{In~\cite{koolen2012putting} the algorithm is presented in
  terms of probabilities with the log loss. Here we give the update generalized to
$(c,\eta)$-realizable losses.}  of these weights after
receiving the true label $\yt$ is given
by
${\ati[t+1] =
  \Pww\frac{
    \ati e^{-\eta\lti} (\sumin[j]\atj)
  }{
    \sum_{j=1}^{\nexperts}\atj e^{-\eta\lti[j]}
  }  + \Psw\sti
}$,
and
${\sti[t+1] =
  \Pws \frac{
    \ati e^{-\eta\lti} (\sumin[j]\atj)
  }{
    \sum_{j=1}^{\nexperts}\atj e^{-\eta\lti[j]}
  }
  + \Pss\sti 
}$ for $i=1,\ldots,\nexperts$.
Recall that the authors of~\cite{koolen2012putting} proved that an MPP
mixing scheme implicitly induces a prior over partition
specialists. The following states that the Markov 
prior is induced by~\eqref{eq:new-mixing-scheme}.
\begin{proposition}
  \label{prop:pbts-as-mpp}
  Let $0< \alpha < 1$, and $0< \theta <1$.
  Then the partition specialists algorithm with Markov prior
  parameterized with $\Psw=\theta$, $\Pws=\alpha$, 
  $\piw=\frac{\theta}{\alpha+\theta}$, and $\pis=\frac{\alpha}{\alpha+\theta}$
  is equivalent to \shareth\ parameterized with $\alpha$ and $\theta$.
\end{proposition}
The proof (given in Appendix~\ref{sec:proof-pbts-as-mpp-thm}) amounts
to showing for all $t$ that $\frac{\bat}{\piw}=\bwt$ and
$\frac{\bst}{\pis}=\bvt$.
The Markov prior on partition specialists therefore corresponds to a
geometrically-decaying MPP mixing scheme!
Note however that  we have proved a better regret bound for this algorithm in
Theorem~\ref{thm:pods-bound}.

\section{Discussion}\label{sec:discussion}
We gave an efficient projection-based algorithm for tracking experts
with memory for which we proved the best known regret bound. We also
gave an algorithm to compute relative entropy projection onto the
simplex with non-uniform (lower) box constraints exactly in
$\cO(\nexperts)$ time, which may be of independent interest.
We showed that the weight-sharing equivalent of our projection-based
algorithm is in fact a geometrically-decaying mixing scheme for
{\em Mixing Past Posteriors}~\cite{bousquet2002tracking}. Furthermore
we showed that this mixing scheme corresponds exactly
to the previous best known result (the partition specialists algorithm with
Markov prior~\cite{koolen2012putting}), and we therefore improved their bound.
We proved a guarantee favoring projection updates over
weight-sharing when updating weights may incur costs, such as in
portfolio optimization with proportional transaction costs. We are
currently applying \podsth\ to this problem, primarily extending the
work of~\cite{singer1997switching} in the sense of incorporating both
the assumption of ``memory'' and transaction costs.

In this work we focused on proving good regret bounds, which naturally
required optimally-tuned parameters. A limitation of our work is that
in practice the optimal parameters are unknown. 
This is a common issue in online learning, and one may
employ standard techniques to address this such as the ``doubling
trick'', or by using a Bayesian mixture over parameters~\cite{vovk1999derandomizing}.
For a prominent recent result in this area see~\cite{jun2017improved}.

Finally, the work of~\cite{koolen2012putting} gave a Bayesian interpretation to
MPP, however this is lost when one uses the projection update of
PoDS. We ask: Is there also a Bayesian interpretation to these projection-based updates?

\paragraph{Ethical considerations.} While the scope of applicability of
online learning algorithms is wide, this research in regret-bounded
online learning is foundational in nature and we therefore cannot foresee the extent
of any societal impacts (positive or negative) this research may have.

\begin{ack}
This work was supported by the UK Engineering and Physical Sciences Research Council (EPSRC) grant EP/N509577/1.
\end{ack}

\bibliographystyle{abbrv}
\bibliography{SWMbib}

\newpage
\appendix
\section{Proof of Theorem~\ref{thm:projection-form}}\label{sec:proj-linear-time-proof}
A note on the proof: The proof of the theorem follows very closely to
the proof of Theorem $7$ in~\cite{herbster2001tracking} (including Claims 1, 2, and 3).
There the problem is concerned with uniform constraints, whereas  we consider
non-uniform constraints. In particular Claims~\ref{claim:ordering}
and~\ref{claim:ksetonly} given below are
generalizations of Claims 2 and 3 of~\cite{herbster2001tracking}.
The proof of the second statement of
Theorem~\ref{thm:projection-form} is almost identical to the proof of
Theorem 7 in~\cite{herbster2001tracking}. 
We first give a sketch of the proof of the two statements of
Theorem~\ref{thm:projection-form}.

For the first statement, recall that $\Psi:=\{i\in [n] : \pii=\betai\}$ is the set of indexes 
of components which must be set to their constraint values.
To prove the first statement we will show that given $\bw$ and $\Cb$,
each component of the point $\projw$ either takes the value of its
lower box constraint, $\betai$, or is equal to $\wi$ multiplied by a
factor $\lambda$, with 
\begin{equation*}
  \lambda = \frac{1-\sum_{i\in\Psi}\betai}{1 -
    \sum_{i\in\Psi}\wi}\,.
\end{equation*}
We then argue that each component $\pii=\max{\{\betai;\lambda\wi\}}$ for $i=1,\ldots,n$.

For the second statement, we first show that $\Psi$, which uniquely
specifies $\projw$, is the set of
minimum cardinality such that when all other components are
re-normalized, no component lies below its constraint value, and then
show that Algorithm~\ref{alg:linear-time-proj} finds this set in
$\cO(n)$ time.

\begin{proof}[Proof of the first statement of Theorem~\ref{thm:projection-form}]
  Recall the first statement of the theorem: that $\projw$
  takes the form~\eqref{eq:proj-max-form}.
  Given $\bw$ and the non-empty set $\Cb$, the point $\projw$ is the 
  minimizer of the following convex optimization problem  
  \begin{equation}
    \label{eq:opt-prob}
    \begin{aligned}
      &\min_{\bu} &\quad& \quad\,\,\RE{\bu}{\bw} \\
      &\,\,\,\textrm{s.t.} &   &\,\,\,\, \betai -\ui \leq 0, && i=1,\ldots,n \\
      & & &\bm{1}\cdot\bu - 1 = 0 &&\,.
    \end{aligned}
  \end{equation}
  Since $\RE{\bu}{\bw}$ is convex in its first argument, and $\Cb$ is
  a convex set, then~\eqref{eq:opt-prob} has a unique minimizer, which
  we denote by $\bp$.
  
  Constructing the Lagrangian of~\eqref{eq:opt-prob} with Lagrange
  multipliers $\bm{\xi}\succeq \bm{0}, \nu\in\mathbb{R}$,
  \begin{equation*}
    \cL(\bu,\bm{\xi},\nu) =
    \sumin\ui\ln{\frac{\ui}{\wi}} + \bm{\xi}^\top (\bbeta-\bu) +
    \nu(\bm{1}\cdot\bu - 1)\,, 
  \end{equation*}
  and setting $\nabla_{\bu} \cL(\bu,\bm{\xi},\nu) = \bm{0}$ gives for $i=1,\ldots,n$,
  \begin{equation*}
    \frac{\partial \cL}{\partial \ui} = \ln{\frac{\ui}{\wi}} + 1
    -\xi_i +\nu = 0\,.
  \end{equation*}
  This then gives for $i=1,\ldots,n$,
  \begin{equation*}
    \pii = \wi e^{\xi_i - 1 - \nu}\,.
  \end{equation*}
  Since $\RE{\bu}{\bw}$ is convex in its first argument,
  and~\eqref{eq:opt-prob} has only linear constraints then strong
  duality holds and we may exploit the complementary slackness
  Karush-Kuhn-Tucker necessary condition of the optimal solution (see
  e.g.,~\cite[Chapter 5]{boyd2004convex}). That 
  is, $\xi_{i}(\betai - \pii) =0$ for all $i=1,\ldots,n$.
  Therefore for any $i$ such that $\pii > \betai$, the
  corresponding Lagrange multiplier is zero, and we have
  \begin{equation*}
    \pii = \wi e^{-1-\nu}\,.
  \end{equation*}
  Recall $\Psi=\{i : \pii = \betai\}$,  we then have 
  \begin{equation*}
    1 = \sumin\pii = \sum_{i\in\Psi}\pii + \sum_{i\in[n]\setminus\Psi}\pii
    = \sum_{i\in\Psi}\betai + \sum_{i\in[n]\setminus\Psi}\wi e^{-1-\nu}\,.
  \end{equation*}
  Re-arranging gives
  \begin{equation*}
    e^{-1-\nu} = \frac{1-\sum_{i\in\Psi}\betai}{\sum_{i\in[n]\setminus\Psi}\wi} =
    \frac{1-\sum_{i\in\Psi}\betai}{1 - \sum_{i\in\Psi}\wi}\,.
  \end{equation*}
  Therefore for each index $i\in [n]$, either $i$ is in $\Psi$ which
  implies $\pii=\betai$, or $i\notin\Psi$ and therefore $\pii =
  \lambda\wi$, where
  \begin{equation*}
    \lambda=\frac{1-\sum_{j\in\Psi}\betai[j]}{1 - \sum_{j\in\Psi}\wi[j]}\,.
  \end{equation*}
 
  We now establish that $\pii=\max{\{\betai;\lambda\wi\}}$ for all
  $i=1,\ldots,n$. Observe that if $i\in\Psi$, then $\pii = \wi
  e^{\xi_i - 1 - \nu} = \betai$, and since the Lagrange multiplier
  $\xi_{i}\geq 0$ then $\pii \geq \wi e^{ - 1 - \nu} = \lambda\wi$.

  For $i\notin\Psi$, then this implies
  $\pii=\lambda\wi>\betai$,  since if $\pii = \betai$ then $i\in\Psi$, and
  if $\pii < \betai$ then we have a contradiction since $\bp$ is
  not a feasible solution to~\eqref{eq:opt-prob}. We therefore conclude that
  $\bp$ is such that for all $i=1,\ldots,n$,
  \begin{equation*}
    \pii = \max\left\{\betai; \frac{1-\sum_{j\in\Psi}\betai[j]}{1 -
        \sum_{j\in\Psi}\wi[j]}\wi\right\}\,,
  \end{equation*}
  which completes the proof of the first statement of the Theorem.
\end{proof}

The proof of the second statement of Theorem~\ref{thm:projection-form}
will rely on the following two claims.
\begin{claim}\label{claim:ordering}
  Given $\bw$ and $\bbeta$, let
  $\br:=\bw\odot\frac{1}{\bbeta}$. Without loss of generality, for
  $i<j$ assume $\ri\leq\rj$. Let $\lambda = \frac{1-\sum_{i\in\Psi}\betai}{1 -
    \sum_{i\in\Psi}\wi}$, then
  \begin{equation}
    \label{eq:claim2}
    \bp = \left(\betai[1],\ldots,\betai[\vert\Psi\vert],\lambda\wi[\vert\Psi\vert+1],\ldots,\lambda\wi[n]\right)\,.
  \end{equation}
\end{claim}
\begin{proof}
  In the proof of the first statement of Theorem~\ref{thm:projection-form} we
  established that $\bp$ is a permutation
  of~\eqref{eq:claim2}, that is, either $\pii=\betai$ or
  $\pii=\lambda\wi$ for $i=1,\ldots,n$. We also established that
  $\pii=\max{\{\betai; \lambda\wi\}}$ for $i=1,\ldots,n$.

  Suppose $\bp$ is not in the form of~\eqref{eq:claim2}. Then
  there exists $a<b$ such that $\pii[a]=\lambda\wi[a]$ and
  $\pii[b]=\betai[b]$ (that is, $b\in\Psi$ and $a\notin\Psi$).

  If $\pii[a]=\lambda\wi[a]$ then by the first statement of
  Theorem~\ref{thm:projection-form} we have $\lambda\wi[a]>\betai[a]$.
  However since $\ri[a]\leq\ri[b]$, and $\lambda>0$, this implies
  $\frac{\lambda\wi[a]}{\betai[a]}\leq\frac{\lambda\wi[b]}{\betai[b]}$.
  We then have
  $1<\frac{\lambda\wi[a]}{\betai[a]}\leq\frac{\lambda\wi[b]}{\betai[b]}$, which
  implies $\lambda\wi[b]>\betai[b]$. However we necessarily assumed
  that $\pii[b]=\betai[b]$. This violates the first statement of
  Theorem~\ref{thm:projection-form} that  $\pii[b]=\max{\{\lambda\wi[b], \betai[b]\}}$, and thus contradicts our
  assumption that $\bp$ is the minimizer
  of~\eqref{eq:opt-prob}. Hence our supposition that
  $\bp$ is not in the form of~\eqref{eq:claim2} is false.
\end{proof}

\begin{claim}\label{claim:ksetonly}
  Let $\Psi'=\{1,\ldots,k\}$, and 
  $\Psi''=\{1,\ldots,k+1\}$, and let $\lambda' = \frac{1-\sum_{i\in\Psi'}\betai}{1 -
    \sum_{i\in\Psi'}\wi}$, and ${\lambda'' = \frac{1-\sum_{i\in\Psi''}\betai}{1 -
    \sum_{i\in\Psi''}\wi}}$. Then let
  \begin{equation*}
    \bu' = \left(\overbrace{\betai[1],\ldots,\betai[\vert\Psi'\vert]}^{k},\lambda'\wi[\vert\Psi'\vert+1],\ldots,\lambda'\wi[n]\right)\,,
  \end{equation*}
  and
  \begin{equation*}
    \bu'' = \left(\overbrace{\betai[1],\ldots,\betai[\vert\Psi''\vert]}^{k+1},\lambda''\wi[\vert\Psi''\vert+1],\ldots,\lambda''\wi[n]\right)\,,
  \end{equation*}
then $\RE{\bu'}{\bw}\leq\RE{\bu''}{\bw}$.
\end{claim}
\begin{proof}
    Consider the following convex optimization problem for some $\bw\in\relint{\simplex}$,
  \begin{equation}
    \label{eq:opt-prob2}
    \begin{aligned}
      &\underset{\bu}{\text{min}}&\quad &\,\,\,\,\,\RE{\bu}{\bw} \\
      &\,\,\text{s.t.} &   & \,\,\,\,\betai - \ui = 0, && i=1,\ldots,k\\
      &&&\bm{1}\cdot\bu - 1 = 0&&\,.
    \end{aligned}
  \end{equation}
  The point $\bu'$ is the unique minimizer of~\eqref{eq:opt-prob2},
  while $\bu''$ clearly also satisifies the constraints
  of~\eqref{eq:opt-prob2} and is therefore a feasible solution.
  This implies that $\RE{\bu'}{\bw} \leq
  \RE{\bu''}{\bw}$.
\end{proof}

\begin{proof}[Proof of the second statement of Theorem~\ref{thm:projection-form}]
  Recall the second statement of the theorem: that
  Algorithm~\ref{alg:linear-time-proj} computes $\projw$ in linear
  time. We prove this statement by first showing that the set $\Psi$
  corresponding to this projection is the set of components 
  of minimal cardinality to set to their constraint values such that
  when the other components are normalized, no component lies below
  its constraint value.
  We then prove that Algorithm~\ref{alg:linear-time-proj} computes the
  projection by finding this set in linear time.

  In the proof of the first statement of the theorem  we proved that
  $\bp$ has the form~\eqref{eq:proj-max-form}.  Thus $\bp$ is uniquely
  specified by the set $\Psi=\{i\in [n] : \pii = \betai\}\subseteq\{1,\ldots, n\}$. There are
  therefore $2^{n}$ possible solutions.
  Claim~\ref{claim:ordering} proves that the magnitude of the ratio of
  a component and its constraint is smaller for a component
  to be set to its constraint value than a component to be normalized.
  That is, if $i\in\Psi$ and $j\notin\Psi$, then
  $\frac{\wi}{\betai}\leq\frac{\wi[j]}{\betai[j]}$. This reduces the
  number of feasible solutions to $n$.
  
  Given these $n$ possible solutions, claim~\ref{claim:ksetonly} shows
  that if $\Psi'\subseteq\Psi''$ with corresponding candidate
  projection vectors $\bu'$ and $\bu''$ respectively, then $\RE{\bu'}{\bw} \leq
  \RE{\bu''}{\bw}$. Thus to compute the projection, one must find the
  set $\Psi$ of minimum cardinality whose corresponding candidate
  projection vector is in $\Cb$.

  Observe that this ``minimal'' set $\Psi$ is specified uniquely by a
  threshold, $\threshold$, such that $\Psi = \{i\in[n] :
  \ri<\threshold\}$, where $\ri=\frac{\wi}{\betai}$, for
  $i=1,\ldots,n$. Algorithm~\ref{alg:linear-time-proj} finds $\Psi$ by
  finding this threshold.   The algorithm initially computes the
  vector $\br=\bw\odot\frac{1}{\bbeta}$ and when $\threshold$ has been found, the
  algorithm sets all components of $\wi$ where $\ri<\threshold$ to their
  thresholds $\betai$, and normalizes the remaining components.

  We now discuss how the algorithm finds $\threshold$ in linear time.
  On each iteration a candidate threshold is examined. These candidate
  thresholds are determined from an index set $\mW$, which is
  initially set to $\{1,\ldots,n\}$.
  On each iteration the threshold $\threshold$ is chosen as the median
  of the ratios in the set $\{ \ri : i \in \mW\}$
  (line~\ref{al:median}).
  This can be done in $\cO(\vert\mW\vert)$
  time~\cite{blum1973time}. If $\vert\mW\vert$ is even, then the
  algorithm can choose between the $\frac{\vert\mW\vert}{2}$ and the
  $\frac{\vert\mW\vert+1}{2}$ largest element arbitrarily. 
  The set $\mW$ is then sorted into
  two sets, $\mL$ and $\mH$, where ${\mL=\{i \in \mW : \ri < \threshold\}}$
  and ${\mH=\{i \in \mW : \ri > \threshold\}}$.

  The normalizing constant $\lambda$ is then computed
  (line~\ref{al:lambda}). If $\lambda\threshold<1$, then by
  Claims~\ref{claim:ordering} and~\ref{claim:ksetonly} the true
  threshold must be larger than the current candidate threshold
  $\threshold$, and must therefore correspond to $\ri$ for an index $i$
  contained in $\mH$. Otherwise the true threshold must be either equal to the
  current candidate threshold, or must correspond to $\ri$ for an index $i$
  contained in $\mL$.

  Since $\threshold$ was taken to be the median, then the algorithm
  iterates this procedure, setting $\mW=\mL$ or $\mW=\mH$ as appropriate.
  Additionally, since $\threshold$ was taken to be the median, then
  ${\max{\{\vert\mL\vert;\vert\mH\vert\}}\leq\frac{1}{2}\vert\mW\vert}$.
  When $\mW=\emptyset$, then the algorithm has found $\threshold$, and
  the projection is computed.

  There are a maximum of $\lceil\log{n}+1\rceil$ iterations of lines~\ref{al:whilestart}-\ref{al:whileend}, with the
  $i^{th}$ iteration taking $\cO(\frac{n}{2^{i}})$ time. The algorithm
  therefore takes $\cO(n)$ time to find $\threshold$, and the time
  complexity of the algorithm is therefore $\cO(n)$.
\end{proof}

\newpage  
\section{Proof of Corollary~\ref{cor:no-switch}}\label{sec:cor-proof}
\begin{proof}
  Let $\Psi := \{i\in[n] : \pii=\betai\}$. Recall from
  Theorem~\ref{thm:projection-form} that the projected vector $\bp$ takes the 
  form~\eqref{eq:proj-max-form}. Expanding the relative entropy terms of~\eqref{eq:pods-no-switch}
  then gives the following, 
  \begin{align*}
    \RE{\bu}{\bw} - \RE{\bu}{\bp}
    &=\sumin\ui\ln{\left(\frac{\pii}{\wi}\right)}\\
    &\geq\sumin\ui\ln{\left(\frac{\left(1 - \sum_{j\in \Psi}\betai[j]\right)\wi}{\left(1 -
            \sum_{j\in\Psi}\wi[j]\right)\wi}\right)}\\
    &=\ln{\left(\frac{1 - \sum_{j\in \Psi}\betai[j]}{1 -
      \sum_{j\in\Psi}\wi[j]}\right)}\\
    &\geq\ln{(1-\alpha)}\,,\label{eq:saved-cost}
  \end{align*}
  where the first inequality follows from the definition of $\pii$
  in~\eqref{eq:proj-max-form} and the fact that $\max\{a,b\}\geq b$.
  The second inequality follows from the fact that
  $\sum_{j\in\Psi} \wi[j]\geq 0$ and
  $\sum_{j\in\Psi}\betai[j]\leq\alpha$.
\end{proof}

\section{Proof of Theorem~\ref{thm:pods-bound}}\label{sec:thm-proof}
\begin{proof}
  We first prove the bound for \podsth, and then prove that \shareth\ has
  the same bound.
  We use the relative entropy $\RE{\but}{\bwt}$ as a measure of
  progress of the algorithm, where $\but$ is a
  comparator vector which we take to be a basis vector $\bei$ for some
  $i\in [\nexperts]$ corresponding to the locally best expert $\iit$ in 
  hindsight on trial $t$. Recall that the comparator sequence
  $\iit[1],\ldots,\iit[T]$ is partitioned  with $\nswitches$ {\em
    switches} into $\nsegments$ {\em segments}, where a segment is
  defined as a sequence of trials where the comparator is unchanged,
  i.e. $\iit[a]=\ldots=\iit[b]$ for some $a<b$.

  Recall that \pred\ and $\ell$ are assumed to be
  $(c,\frac{1}{c})$-realizable. That is, for any $\bwt\in\simplex$,
  $\bxt\in\cD^{\nexperts}$, and $\yt\in\cY$, there exists $\eta>0$
  such that
  \begin{equation}
    \label{eq:c-eta-realizability}
    \ell(\textup{\pred}(\bw,\bx), y) \leq -c\ln{\sumin \vi e^{-\eta \ell(\xii,y)}}
  \end{equation} 
  holds with $c\eta=1$.
  
  We first establish that 
  \begin{equation}
    \label{eq:loss-inequality}
    \lt - \ltit \leq c\left(\RE{\but}{\bwt} - \RE{\but}{\bwtdot}\right)
  \end{equation}
  holds for all $t$. Expanding the relative entropy terms gives
  \begin{align*}
    \RE{\but}{\bwt} - \RE{\but}{\bwtdot}
    &= \sumin\uti\ln{\frac{\wtdoti}{\wti}}\\
    &= \sumin\uti\ln{\frac{\wti e^{-\eta\lti}}{\wti \sumin[j]\wtj
      e^{-\eta\lti[j]}}}\\
    &= -\eta\sumin\uti\lti - \ln{\sumin[j]\wtj e^{-\eta\lti[j]}}\\
    &\geq -\eta\ltit  + \frac{1}{c}\lt\,,
  \end{align*}
  where the inequality follows from~\eqref{eq:c-eta-realizability}. Multiplying both sides by $c$ gives~\eqref{eq:loss-inequality}.
  
  We now find lower bounds, $\delta$, for $\RE{\but}{\bwtdot} -
  \RE{\but[t+1]}{\bwt[t+1]}$ to give non-negative terms of the form $\RE{\but}{\bwtdot} -
  \RE{\but[t+1]}{\bwt[t+1]} - \delta \geq 0$, which we will multiply
  by $c$ and add
  to~\eqref{eq:loss-inequality} to give a telescoping sum of relative
  entropy terms. We consider three distinct cases for the
  different values of $\but$ over the $T$ trials. 

  For the first case, we consider when there is no switch immediately after trial $t$
  (i.e., $\but = \but[t+1]$). We
  use Corollary~\ref{cor:no-switch} with $\bu=\but$, $\bw=\bwtdot$, and
  $\bbeta=\bbetat$. It follows then by definition that $\bp=\bwt[t+1]$
  and we obtain
  \begin{equation}
    \label{eq:eq-no-switch-1}
    \RE{\but}{\bwtdot} - \RE{\but[t+1]}{\bwt[t+1]} \geq \ln{(1-\alpha)}\,,
  \end{equation}
  which gives a telescoping sum of relative entropy terms within in each segment, paying
  $c\ln(1/(1-\alpha))$ for every trial where $\but=\but[t+1]$.
  
  For the two remaining cases, we will consider the segment boundaries, that is, the case when there is a
  switch and $\but\neq\but[t+1]$. 
  W.l.o.g let $\but=\bei[j]$ and let $\but[t+1]=\bei[k]$ for any
  $j\neq k$ (that is we switch from expert ``$j$'' to expert ``$k$'' after trial
  $t$). We then have the following 
  \begin{equation}\label{eq:segment-boundary}
    \RE{\but}{\bwtdot} - \RE{\but[t+1]}{\bwt[t+1]}
    =\sum\limits_{i=1}^{\nexperts}\uti\ln{\frac{\uti}{\wtdoti}} - 
    \sum\limits_{i=1}^{\nexperts}\uti[t+1]\ln{\frac{\uti[t+1]}{\wti[t+1]}}
    =\ln{\frac{1}{\wtdotj}} + \ln{\wtk[t+1]}\,,
  \end{equation}
  thus we collect a $\ln{(1/\wtdotj)}$ term from the \textit{last}
  trial of the segment of expert $j$ and a $\ln{(\wtk[t+1])}$ term from
  the \textit{first} trial of the new segment of expert $k$.
  We now consider the remaining two cases: when trial $t+1$ is the first time
  expert $k$ predicts well, and when trial $t+1$ is a trial on
  which we ``re-visit'' expert $k$. 

  For the first of these two cases, we consider the first time expert
  $k$ starts to predict well. We then use~\eqref{eq:update-proj} and~\eqref{eq:update-beta} to give
  \begin{equation}
    \label{eq:first-k-beta}
    \ln{\wtk[t+1]}\geq \ln{\betatk} \geq
    \ln{((1-\theta)^{t-1}\betatk[1])} =
    \ln{\left((1-\theta)^{t-1}\frac{\alpha}{\nexperts}\right)}\,.
  \end{equation}
  Substituting~\eqref{eq:first-k-beta} into~\eqref{eq:segment-boundary},
  we therefore pay $-c\ln{((1-\theta)^{t-1}\frac{\alpha}{n})}$ to switch
  to a new expert for the first time on trial $t+1$.

  Finally for the second of these two cases, we consider when expert
  $k$ has predicted well before. Let trial
  $q<t$ denote the \textit{last} trial of expert $k$'s most recent
  ``segment''. We then have the following (again
  using~\eqref{eq:update-proj} and~\eqref{eq:update-beta}),
  \begin{equation}
    \ln{\wtk[t+1]}
    \geq \ln{\betatk}
    \geq \ln{((1-\theta)^{t-q-1}\betatk[q+1])}
    \geq \ln{((1-\theta)^{t-q-1}\alpha\theta\wtdotk[q])}\,.\label{eq:not-first-k-beta}
  \end{equation}
  By substituting~\eqref{eq:not-first-k-beta}
  into~\eqref{eq:segment-boundary} for each segment boundary, 
  and summing over these boundaries, we therefore pay
  $-c\ln{((1-\theta)^{t-q-1}\alpha\theta)}$ in order to telescope the 
  $\ln{(\wtdotk[q])}$ term with the $\ln{(1/\wtdotk[q])}$ term from
  the end of expert $k$'s most recent segment ending on trial $q$.
  
  Putting these together we thus pay
  $c\ln{(1/(1-\alpha))}$ for every trial on which we don't switch
  (from Corollary~\ref{cor:no-switch}), we pay $c\ln{(1/(1-\theta))}$ for 
  every expert in our pool that \textit{isn't} predicting well or
  involved in a switch on every trial (i.e., $\pool-1$ times, on non-switch
  trials, and $\pool-2$ times on switch trials,
  from~\eqref{eq:first-k-beta} and~\eqref{eq:not-first-k-beta}), and finally when we
  switch to an expert $k$ before trial $t+1$ we pay
  $c\ln{(\nexperts/\alpha)}$ if it is the first 
  time to track expert $k$ (there are $\pool-1$ such trials), and
  $c\ln{(1/\alpha\theta)}$ otherwise (there are $\nswitches-\pool+1$
  such trials).
  
  Summing over all trials, and using $\RE{\but[1]}{\bwt[1]}\leq\ln{\nexperts}$
  then gives
  \begin{align}
    \sumtT\lt - \sumtT\ltit
    &\leq \sumtT c\left( \RE{\but}{\bwt} -
      \RE{\but}{\bwtdot} + \RE{\but}{\bwtdot} -
      \RE{\but[t+1]}{\bwt[t+1]}\right)\nonumber\\
    &\leq c\RE{\but[1]}{\bwt[1]} + c(T-\nswitches
      -1)\ln{\left(\frac{1}{1-\alpha}\right)} + c(\pool-1) \ln{\left(\frac{\nexperts}{\alpha}\right)}\nonumber\\
    &\qquad + c((\pool -1)(T-1)-\nswitches)
      \ln{\left(\frac{1}{1-\theta}\right)} +
      c(k-\pool+1)\ln{\left(\frac{1}{\alpha\theta}\right)}\nonumber\\
    &\leq c\pool\ln{\nexperts} + c(T-\nswitches
      -1)\ln{\left(\frac{1}{1-\alpha}\right)} +
      ck\ln{\left(\frac{1}{\alpha}\right)}\nonumber\\
    &\qquad + c((\pool
      -1)(T-1)-k)\ln{\left(\frac{1}{1-\theta}\right)}
      +
      c(\nswitches -\pool+1)\ln{\left(\frac{1}{\theta}\right)}\label{eq:untuned-bound}\,.
  \end{align}
  The optimal tuning of $\alpha$ and $\theta$ that
  minimizes~\eqref{eq:untuned-bound} is given by $\alpha =
  \frac{\nswitches}{T-1}$ and $\theta=\frac{\nswitches-\pool+1}{(\pool-1)(T-2)}$.
  Substituting these values into~\eqref{eq:untuned-bound} gives a
  bound of 
  \begin{equation*}
    cm\ln{\nexperts} + c(T-1)\cH\left(\frac{\nswitches}{T-1}\right) +
    c(\pool -1)(T-2)\cH\left(\frac{\nswitches-\pool+1}{(\pool -1)(T-2)}\right)\,,
  \end{equation*}
  which completes the proof for \podsth.

  We now prove that \shareth\ has the same bound with an almost
  identical argument as the proof just given for \podsth.
  Firstly observe that~\eqref{eq:segment-boundary} is independent of the
  algorithm update and therefore holds for both
  algorithms. Additionally, observe that the proof for \podsth\ relies on the 
  inequalities~\eqref{eq:loss-inequality},~\eqref{eq:eq-no-switch-1},~\eqref{eq:first-k-beta},
  and~\eqref{eq:not-first-k-beta}. We now prove that these
  inequalities hold for \shareth, and thus the two algorithms share
  the same bound.

  Firstly we observe that inequality~\eqref{eq:loss-inequality} holds
  since both algorithms use the same loss update, and we assume that the prediction function and loss
  function are $(c,\frac{1}{c})$-realizable.

  Secondly, it follows directly from the update~\eqref{eq:update-wt}
  that~\eqref{eq:eq-no-switch-1} holds for \shareth\ when
  $\but=\but[t+1]$, since $\bwt[t+1] \geq(1-\alpha)\bwtdot$ and
  therefore
  \begin{equation*}
    \RE{\but}{\bwtdot} - \RE{\but[t+1]}{\bwt[t+1]} =
    \sumin\uti\ln{\frac{\wti[t+1]}{\wtdoti}}\geq\sumin\uti\ln{\frac{(1-\alpha)\wtdoti}{\wtdoti}}=\ln{(1-\alpha)}\,.
  \end{equation*}

  The proof that~\eqref{eq:first-k-beta} holds follows directly from the
  updates~\eqref{eq:update-wt} and~\eqref{eq:update-vt} and the fact
  $\bvt[1]=\frac{\bm{1}}{\nexperts}$. That is, for the first time expert
  ``$k$'' appears on trial $t+1$,
  \begin{equation*}
    \ln{\wtk[t+1]}\geq \ln{(\alpha\vtk)} \geq
    \ln{((1-\theta)^{t-1}\alpha\vtk[1])} =
    \ln{\left((1-\theta)^{t-1}\frac{\alpha}{\nexperts}\right)}\,.
  \end{equation*}

  Similarly, the proof that~\eqref{eq:not-first-k-beta} holds follows
  directly from the updates~\eqref{eq:update-wt}
  and~\eqref{eq:update-vt}. That is, when we return to expert ``$k$'' on
  trial $t+1$,
  \begin{equation*}
    \ln{\wtk[t+1]}
    \geq \ln{(\alpha\vtk)}
    \geq \ln{((1-\theta)^{t-q-1}\alpha\vtk[q+1])}
    \geq \ln{((1-\theta)^{t-q-1}\alpha\theta\wtdotk[q])}\,.
  \end{equation*}
  Having shown that the inequalities~\eqref{eq:loss-inequality},~\eqref{eq:eq-no-switch-1},~\eqref{eq:first-k-beta},
  and~\eqref{eq:not-first-k-beta} hold for \shareth, the remainder
  of the proof follows exactly as the proof for \podsth.
\end{proof}

\section{Proof of Theorem~\ref{thm:proj-lt-sharing-general}}\label{sec:thm-proj-lt-sharing-proof}
Before proving Theorem~\ref{thm:proj-lt-sharing-general}, we introduce
some additional notation. Let $\bm{p}:=\projw[\Cb]$, and recall the 
definition of $\bw' = (1-\alpha)\bw+\alpha\bv$. We then define the following sets,
\begin{equation*}
  \begin{aligned}
    &\Pinc := \{i\in [n] : \pii > \wi\}\,,
    &&\Pdec := \{i\in [n] : \pii \leq \wi\}\,,\\
    &\Sinc := \{i\in [n] : \wi' > \wi\}\,,
    &&\Sdec := \{i\in [n] : \wi' \leq \wi\}\,.
  \end{aligned}
\end{equation*}
The subscripts $inc$ and $dec$ correspond to the relative change in
the weights before and after the corresponding update - whether they
\textit{increase} or \textit{decrease}, respectively.

We first require the following corollary, which follows naturally from Theorem~\ref{thm:projection-form}.
\begin{corollary}
  \label{cor:pii-equals-betai}
If $i\in\Pinc$ then $\pii=\betai$.
\end{corollary}
\begin{proof}
  Recall that Theorem~\ref{thm:projection-form} states that $\bp$ is such that for  $i=1,\ldots,\nexperts$, 
  \begin{equation*}
    \pii=\max{\left\{\betai;\lambda\wi\right\}}\,,
  \end{equation*}
  where $\lambda=\frac{1 - \sum_{j\in \Psi}\betai[j]}{1 -
      \sum_{j\in\Psi}\wi[j]}$ is a normalizing constant.
  We first establish that $\lambda\leq 1$. Suppose $\lambda>1$, then
  this implies $\sum_{i\in\Psi}\wi >\sum_{i\in\Psi}\betai$. In this
  case there must exist $i\in\Psi$ such that
  $\wi>\betai$. However if $\lambda>1$ then $\lambda\wi>\wi>\betai$,
  but since $i\in\Psi$ then $\pii=\betai$, which must be greater than
  $\lambda\wi$ by Theorem~\ref{thm:projection-form}. This leads to a
  contradiction and thus our supposition that $\lambda>1$ is false.  
  
  The form of $\bp$ implies that $i\in\Pinc$ iff  $\wi
  <\betai$, since if $\wi\geq\betai$ then this implies that either
  $\pii=\betai\leq\wi$ or $\pii=\lambda\wi\leq\wi$, and in both of
  these cases $i$ must be in $\Pdec$. It then follows that if $i\in\Pinc$
  then $\pii=\betai$ since otherwise $\pii=\lambda\wi\leq\wi<\betai$ which is a contradiction.  
\end{proof}

We now require the following two lemmas, the first states that if a
weight $\wi$ were to increase after the projection update, then it
would always increase after the weight-sharing update.
\begin{lemma}\label{lem:PincSinc}
  $\Pinc\subseteq\Sinc$.
\end{lemma}
\begin{proof}
  For any $i\in [n]$ we have
  \begin{equation*}
    \wi' - \wi = \left(1 - \alpha \right)\wi +
    \alpha \vi - \wi
    = \alpha\left(v_i - \wi\right)\,,
  \end{equation*}
  and it follows that $i\in\Sinc$ iff $\wi<\vi$.
  Using Corollary~\ref{cor:pii-equals-betai} we conclude that if
  $i\in\Pinc$, then $\wi<\pii=\betai=\alpha\vi<\vi$ 
  and then $i$ must also be in $\Sinc$.
\end{proof}

\begin{lemma}\label{lem:2Pinc}
  $\onenormBig{\bp-\bw} = 2\sum_{i\in\Pinc}(\pii-\wi)$, and $\onenormBig{\bw'-\bw} = 2\sum_{i\in\Sinc}(\wi'-\wi)$. 
\end{lemma}
\begin{proof}
  We prove the first equality by observing that 
  \begin{equation*}
    \onenormBig{\bp-\bw} = \sum_{i=1}^{n}\vert \pii-\wi\vert =
    \sum_{i\in\Pinc}(\pii-\wi) + \sum_{i\in\Pdec}(\wi-\pii)\,,
  \end{equation*}
  and since the total weight does not change after an update (i.e.,
  $\sumin \pii = \sumin \wi$), necessarily 
  we have $\sum_{i\in\Pinc}(\pii-\wi) = \sum_{i\in\Pdec}(\wi-\pii)$.
  Since $\sumin\wi'=\sumin\wi$, the same argument can be used to prove
  the second claim.
\end{proof}

\begin{proof}[Proof of Theorem~\ref{thm:proj-lt-sharing-general}]
  Using Corollary~\ref{cor:pii-equals-betai}, and the definition of
  $\bw'$, we have for $i\in\Pinc$,
  \begin{equation}\label{eq:final-pinc-inequality}
    \wi'-\wi = (1-\alpha)\wi + \alpha\vi-\wi
    =\alpha(\vi-\wi)
    =\betai-\alpha\wi
    =\pii-\alpha\wi
    > \pii-\wi\,,
  \end{equation}
  where the inequality arises from the fact that $\alpha<1$.
  Finally combining this inequality with Lemmas~\ref{lem:PincSinc} and~\ref{lem:2Pinc} gives 
  \begin{align*}
    \onenormBig{\bp-\bw} &= 2\sum_{i\in\Pinc}(\pii-\wi)
    &&(\text{Lemma~\ref{lem:2Pinc}})\\ 
                         &< 2\sum_{i\in\Pinc}(\wi'-\wi)&&(\text{Equation~\ref{eq:final-pinc-inequality}})\\
                         &\leq 2\sum_{i\in\Sinc}(\wi'-\wi)
    &&\text{(Lemma~\ref{lem:PincSinc})}\\ 
                         &=
                           \onenormBig{\bw'-\bw}\,. &&(\text{Lemma~\ref{lem:2Pinc}})
  \end{align*}
\end{proof}

\section{Proof of Proposition~\ref{prop:pbts-as-mpp}}\label{sec:proof-pbts-as-mpp-thm}
\begin{proof}
  It suffices to show that
  \begin{equation}
    \label{eq:normalized-update}
    \frac{\ati}{\sum_{j=1}^{\nexperts}\atj}=\wti\,,
  \end{equation}
  and
  \begin{equation}
    \label{eq:normalized-update-sleep}
    \frac{\sti}{\sum_{j=1}^{\nexperts}\stj}=\vti
  \end{equation}
  for all $t$.
  Since the initial distribution, $\bpi$, of the
  Markov chain prior is taken to be the stationary distribution, the
  detailed balance equation, $\Pws\piw = \Psw\pis$, holds for
  all trials.

  It is therefore straightforward to show that
  $\sum_{i=1}^{\nexperts}\ati=\piw$ and
  $\sum_{i=1}^{\nexperts}\sti=\pis$ for all $t$.
  Letting $\alpha=\Pws$, and $\theta=\Psw$, we proceed to
  prove that~\eqref{eq:normalized-update}
  and~\eqref{eq:normalized-update-sleep} hold simultaneously for all $t$ by induction.
  The case for $t=1$ is trivial. Then by induction on $t$ for $t\geq 1$, 
  \begin{align*}
    \frac{\ati[t+1]}{\piw}
    &= \Pww\frac{\ati e^{-\eta\lti}}{\sum_{j=1}^{\nexperts}\atj
      e^{-\eta\lti[j]}}
      + \frac{\Psw}{\piw}\sti\\
    &= \Pww\frac{\ati e^{-\eta\lti}}{\sum_{j=1}^{\nexperts}\atj
      e^{-\eta\lti[j]}}
      + \frac{\Pws}{\pis}\sti\\
    &= \Pww\wtdoti + \Pws\vti&&\text{(induction)}\\
    &= (1-\alpha)\wtdoti + \alpha\vti\\
    &= \wti[t+1]\,,
  \end{align*}
  and similarly
  \begin{align*}
    \frac{\sti[t+1]}{\pis}
    &= \frac{\Pws\piw}{\pis}\frac{\ati e^{-\eta\lti}}{\sum_{j=1}^{\nexperts}\atj
      e^{-\eta\lti[j]}}
    + \Pss\frac{\sti}{\pis}\\
    &= \Psw\frac{\ati e^{-\eta\lti}}{\sum_{j=1}^{\nexperts}\atj
      e^{-\eta\lti[j]}}
    +\Pss\frac{\sti}{\pis}\\
    &= \Psw\wtdoti + \Pss\vti &&\text{(induction)}\\
    &= \theta\wtdoti + (1-\theta)\vti\\
    &= \vti[t+1]\,.
  \end{align*}
  We therefore conclude by the inductive argument that~\eqref{eq:normalized-update}
  and~\eqref{eq:normalized-update-sleep} hold for all $t\geq 1$.
\end{proof}

\end{document}